\documentclass[11pt]{amsart}
\usepackage{comment}
\usepackage[T1]{fontenc}
\usepackage[latin9]{inputenc}
\usepackage{amsmath}
\usepackage{amsthm}
\usepackage{bm}
\usepackage{cite}
\usepackage{mathalpha}
\usepackage[unicode=true,pdfusetitle,
 bookmarks=true,bookmarksnumbered=false,bookmarksopen=false,
 breaklinks=false,pdfborder={0 0 1},backref=false,colorlinks=true,citecolor=blue]
 {hyperref}
 \usepackage{graphicx}

 \usepackage{xfrac}
 \usepackage{lmodern}

\numberwithin{equation}{section}
\numberwithin{figure}{section}
\theoremstyle{plain}
\newtheorem{thm}{\protect\theoremname}[section]
\theoremstyle{remark}

\theoremstyle{plain}
\newtheorem{lem}[thm]{\protect\lemmaname}
\newtheorem{cor}[thm]{Corollary}

\newtheorem{assumption}{Assumption}

\newtheorem{prop}[thm]{Proposition}

\theoremstyle{definition}
\newtheorem{defn}[thm]{Definition}

\usepackage{xcolor}
\usepackage{mathrsfs,amsthm,amssymb,bbm,fullpage}

\newcommand{\cA}{\mathcal{A}}
\newcommand{\cC}{\mathcal{C}}

\newcommand{\cM}{\mathscr{M}}
\newcommand{\cN}{\mathcal{N}}
\newcommand{\cE}{\mathcal{E}}
\newcommand{\cF}{\mathcal{F}}

\newcommand{\cL}{\mathcal{L}}

\newcommand{\cX}{\mathcal{X}}
\newcommand{\cY}{\mathcal{Y}}

\newcommand{\cP}{\mathcal{P}}
\newcommand{\sP}{\mathscr{P}}

\newcommand{\cS}{\mathcal{S}}
\newcommand{\cZ}{\mathcal{Z}}

\newcommand{\R}{\mathbb{R}}

\newcommand{\E}{\mathbb{E}}

\newcommand{\eps}{\epsilon}

\newcommand{\tensor}{\otimes}

\newcommand{\Ainv}{A^{-1}}

\newcommand{\NN}{{\textsc{nn}}}

\newcommand{\abs}[1]{\lvert#1\rvert}
\newcommand{\norm}[1]{\lVert#1\rVert}

\newcommand{\Cov}{\operatorname{Cov}}
\newcommand{\Var}{\operatorname{Var}}

\newcommand{\Span}{\operatorname{Span}}

\newcommand{\bE}{{\mathbb E}}

\newcommand{\bfG}{{\mathbf G}}

\newcommand{\btheta}{\boldsymbol{\theta}}
\newcommand{\op}{\mathrm{op}}
\newcommand{\tr}{\operatorname{tr}}
\newcommand{\Cat}{\mathrm{Cat}}
\newcommand{\LR}{\textsc{lr}}
\newcommand{\deloc}{\mathscr D}

\newcommand{\bmu}{\boldsymbol{\mu}}
\newcommand{\bu}{\mathbf{u}}

\providecommand{\claimname}{Claim}
\providecommand{\lemmaname}{Lemma}
\providecommand{\theoremname}{Theorem}

\title[Universality of high-dimensional SGD]{Universality of high-dimensional scaling  limits of stochastic gradient descent}

\author{Reza Gheissari}
\author{Aukosh Jagannath}

\address[Reza Gheissari]{Department of Mathematics, Northwestern University}
\email{gheissari@northwestern.edu}

\address[Aukosh Jagannath]{Department of Statistics and Actuarial Science, Department of Applied Mathematics, and Cheriton School of Computer Science, University of Waterloo}
\email{a.jagannath@uwaterloo.ca}

\begin{document}

\begin{abstract}
    We consider statistical tasks in high dimensions whose loss depends on the data only through its projection into a fixed-dimensional subspace spanned by the parameter vectors and certain ground truth vectors. This includes classifying mixture distributions with cross-entropy loss with one and two-layer networks, and learning single and multi-index models with one and two-layer networks. 
    When the data is drawn from an isotropic Gaussian mixture distribution, it is known that the evolution of a finite family of summary statistics under stochastic gradient descent  converges to an autonomous ordinary differential equation (ODE), as the dimension and sample size go to $\infty$ and the step size goes to $0$ commensurately. Our main result is that these ODE limits are universal in that this limit is the same whenever the data is drawn from mixtures of arbitrary product distributions whose first two moments match the corresponding Gaussian distribution, provided the initialization and ground truth vectors are coordinate-delocalized. 
     We complement this by proving two corresponding non-universality results. We provide a simple example where the ODE limits are non-universal if the initialization is coordinate aligned. 
     We also show that the stochastic differential equation limits arising as fluctuations of the summary statistics around their ODE's fixed points are not universal. 
\end{abstract}

\maketitle

\section{Introduction}
Stochastic gradient descent (SGD) and its variants are the go-to optimization methods in machine learning \cite{goodfellow2016deep}. As such, there has been a long history of work analyzing its evolution since its introduction in  \cite{RobMon51}. We will focus on the simplest setting, namely online SGD with i.i.d.\ data (see~\eqref{eq:SGD} for a precise definition). 

From the perspective of classical asymptotic theory---where the dimension is viewed as fixed and the sample size is viewed as going to infinity---the limit theory of SGD is well understood. In this regime, the small-step-size scaling limit of the trajectory is gradient flow on the population loss, i.e., the expected value of the loss function with respect to the data distribution \cite{RobMon51,mcleish,Ljung77}. A rich theory can be be developed for the fluctuations of the trajectory about the gradient flow path, and even large deviations for the trajectory have been studied~\cite{mcleish,kushner1984asymptotic,dupuis1989stochastic,Benaim99}. 

In the modern era, practitioners are interested in fitting complex models with limited access to data. As such has been a tremendous amount of attention regarding the ``high-dimensional'' regime: here the data dimension, parameter dimension, and sample size scale together and one can no longer assume that the step-size is arbitrarily small. In recent years, various forms of high-dimensional scaling limits for the performance of SGD have been developed.

In this paper, we seek to understand to what extent these high-dimensional scaling limits are ``universal''. That is, to what extent these scaling limits are agnostic to the specific properties of the data distribution. Observe that the limit theory in the classical asymptotics regime is very sensitive to the details of the underlying data distribution: The population loss can change even if one varies only high moments of the data distribution. By contrast, our main result is that the high-dimensional scaling limits of online SGD for a general family of statistical tasks are universal and only depend on the data distribution through its first two moments.

Let us now be more precise. 
Many analyses of high-dimensional learning tasks focus on data models with latent low-dimensional structure.  These include
regression tasks, such as spiked matrix and tensor PCA or learning multi-index models with two-layer networks of bounded width, and classification tasks, such as multi-class logistic regression or classification of XOR data by two-layer networks. These problems all have a common structure: The loss at any point in parameter space only depends on the data through its inner products with some finite family of vectors (varying with the point in parameter space).
We formalize these classes of problems as 
\emph{projective models}, see Definition~\ref{def:projective-model}.

Since the dimension is diverging with the number of samples, a trajectory-wise limit theory of SGD does not naively make sense. However, for projective models with isotropic Gaussian (or isotropic Gaussian mixture) data distributions, rotation invariance means the law  of the loss only depends on the point in parameter space through its Gram matrix and its projections into ground truth vectors (and possibly some additional parameters such as second-layer weights for neural networks).
Thus, in these \emph{Gaussian projective models}, 
one can make sense of high-dimensional limits by focusing on the evolution of a fixed-dimensional set of ``summary statistics'' like the Gram matrix and inner products with ground truth vectors. 
Indeed, this reducibility of the law of the loss has been leveraged to great effect in e.g.,~\cite{saad1995dynamics,saad1995line,TanVershyninKaczmarz,TanVershynin,veiga2022phase,goldt2019dynamics,BGJ22,jagannath2025highdimensionallimittheoremssgd,arous2025learningquadraticneuralnetworks,pmlr-v195-arnaboldi23a} to develop high-dimensional limit theorems for SGD (and its variants). 

In particular, using the results of \cite{BGJ22}, it can be shown for Gaussian projective models that the set of summary statistics consisting of the Gram matrix of the parameter vector under SGD and ground truth vectors, asymptotically (as sample size and dimension diverge and the step-size goes to zero proportionately) evolve autonomously by an ordinary differential equation (ODE). 
Furthermore the fluctuations about this ODE, which relate to the escape of SGD from unstable fixed points of the ODE, satisfy an autonomous stochastic differential equation (SDE). This is developed in Section~\ref{sec:Gaussian-projective-models} where we provide explicit formulas for these evolution equations.

In practice,  SGD often exhibits two phases of training, \emph{diffusive phases} where the summary statistics evolve microscopically (include the search phase at the beginning and terminal phase), and \emph{ballistic phases} where they evolve macroscopically (the ballistic phase)~\cite{bottou2003stochastic,BottouLeCun04}. The above-described SDE and ODE limits can be understood as describing the behavior of the important observables of the system under SGD during each of these phases respectively, in the regime of high dimensions with proportionately large number of samples. This has been used to understand sample complexities and probabilities (with respect to the initialization and training dynamics) of succeeding vs.\ failing at the corresponding statistical task, e.g., as in~\cite{TanVershynin,BGJ21,ArousEtAl-multi-spike-SGD}. 

However, when the data distribution is not isotropic Gaussian, this reducibility of the dynamics breaks and makes the problem more challenging. The works~\cite{paquette2021sgd,Paquettes-high-dimensional-notes,braun2025fastescapeslowconvergence,collinswoodfin2025exactdynamicsmulticlassstochastic} have delved into how the dynamics evolve under Gaussian noise but with a non-isotropic covariance profile. In non-Gaussian settings, high-dimensional limit theorems for SGD are even rarer. 
Universal high-dimensional dynamics has been shown in the past decade for Langevin dynamics of spin glasses in~\cite{DLZ-universality-Langevin,DeGh21} and approximate message passing schemes in~\cite{Bayati-et-al-universality,ChenLam21,DudejaLuSen-Universality-semirandom-AMP}. Essential to those papers was the Lipschitz dependence on the data entries. Similarly, the data enters at most quadratically into linear regression and online independent component analysis (ICA), where the analyses of~\cite{BalasubramanianGhosalHe,online-ICA-NIPS,WaMaLu16} did not use Gaussianity. By contrast, for general projective models, the loss function depends on low-dimensional projections of the data in an arbitrarily non-linear fashion. 

In light of the above, the natural question to study is: To what extent do the limits derived in the Gaussian setting hold for the trajectory of summary statistics of online SGD trained on non-Gaussian mixture distribution with the same class means, mixture weights, and in-class covariance? 
Universality of this type was proposed in~\cite{Goldt-PRX-universality,Goldt-et-al-Gaussian-Equivalence-Conjecture} where it was argued that this should follow from a certain ``Gaussian equivalence property'' for the outputs of the first layer weights along training, see Conjecture~1 of~\cite{Goldt-et-al-Gaussian-Equivalence-Conjecture} (see also Section~\ref{subsec:DGEC}).
We prove both of these results for projective models.

\subsubsection*{Our contributions}
Our main results can be summarized as follows. 
We consider projective models whose loss function is thrice-differentiable with derivatives of at most polynomial growth at infinity and whose data distribution is given by a mixture of product measures with enough finite moments.

\begin{itemize}
\item (Ballistic phase) In Theorem~\ref{thm:ballistic-universality}, we show that in the ballistic phase, the ODE scaling limit for the summary statistics is universal and only depends on the first two moments of the mixture components, as long as the initialization and class means are coordinate-delocalized. By coordinate-delocalized we mean that each entry is $O(d^{-\frac{1}{2}+\epsilon})$, as holds for (normalized) i.i.d.\ vectors (see Definition~ \ref{def:delocalized-set}). However, this universality breaks when the initialization is localized on a few coordinates, even in well-studied models such as phase retrieval (see Section~\ref{subsubsec:bad-init-no-ballistic-univ}). 
\item (Diffusive phase) In Theorem~\ref{thm:no-diffusive-universality}, we show that universality generically does not hold for the fluctuations of the summary statistics about fixed points of the ballistic dynamics. Our counter example is a simple single-index model with sub-Gaussian data distribution and coordinate-delocalized initialization.  
\end{itemize}

The main idea of the ballistic universality result is as follows. For projective models, the evolution of the summary statistics is governed by expectations of functions of low-dimensional projections of the data, specifically in the directions of the parameter vectors (and possibly certain ground truth vectors). Random low-dimensional projections of high-dimensional product measures are well-known to satisfy central limit theorems. In particular, we show these expectations are quantitatively close to their Gaussian equivalent, uniformly over projections into coordinate-delocalized directions. The key step is then to ensure that if the initialization and ground truth vectors are coordinate-delocalized, then SGD remains in a delocalized region for all $O(d)$ time scales. We describe the argument in more detail in Section~\ref{subsec:proof-outline}.

\subsection{Setting}
Suppose that we are given a sequence of i.i.d. data $X^1,X^2,\ldots$ each taking values in $\R^d$ with law $\cP_X$ and a loss function $L:\R^p\times\R^d\to\R$, where $\R^p$ is the parameter space. We are interested in analyzing the evolution of online stochastic gradient descent, i.e., the iterative process
\begin{equation}\label{eq:SGD}
    \Theta_\ell = \Theta_{\ell-1}  - \delta \nabla_{\Theta} L(\Theta_{\ell-1}, X^\ell)\,
\end{equation}
where $\delta>0$ is the step-size or \emph{learning rate}.

Our focus will be on the evolution of this online SGD for \emph{projective models}, which we formalize as follows.
Suppose that the data distribution $ \cP_X$ is a $k$-component mixture distribution with class means $\bmu = (\mu^1,...,\mu^k)$ where $\mu^a\in \R^d$ for all $a\in [k]$, weights $(p_a)_{a\in [k]}$, and 
\begin{equation}\label{eq:X-def}
J \sim \Cat((p_a)_{a=1}^k) \qquad \text{and} \qquad X \mid J \sim \mu^J + Y
\end{equation}
where $Y\in \mathbb R^d$ has $Y_i$ i.i.d.\ drawn from some distribution $\nu$ with mean zero and variance $\sigma^2$.\footnote{Here, $\Cat((p_a)_a)$ denotes the categorical distribution $\mathbb P(J =a) = p_a$.}   
We call $\nu$ the \emph{(internal) noise distribution}. When $\nu= \mathcal N(0,\sigma^2)$, this is a Gaussian mixture model (GMM), and for general non-Gaussian $\nu$, we call this the \emph{$\nu$-mixture model}, or $\nu$-MM for short.

To capture supervised and unsupervised learning tasks, as well as multi-layer settings, we suppose that our parameter space splits as $\Theta = (\btheta, w)$ with $\btheta = (\theta^a)_{a\in [k_1]} \in \mathbb R^{d\times k_1}$, and  $w\in \mathbb R^{k_2}$, and we allow for an extra discrete variable $y \in [\cC]$ (e.g., a label). (For a vector $v$ use $\|v\|_q$ to be its $\ell^q$ norm, and when there is no subscript, we mean $q=2$.)

\begin{defn}\label{def:projective-model}
    A statistical model $(L,\cP_X)$ is called a \emph{projective model} if $L(\Theta, X)$ 
    is a function $\psi$ of the inner products $\btheta^\top X = \langle X,\theta^1\rangle ,\ldots,\langle X, \theta^{k_1}\rangle$, $w\in \mathbb R^{k_2}$, and the discrete variable $y$, i.e., 
    \begin{equation}\label{eq:projective-def}
        L(\Theta, X) = \psi(\btheta^\top X,  w ; y)  + \Lambda \|\btheta \|^2\,,
    \end{equation}
    for some $\psi : \mathbb R^{k_1} \times  \mathbb R^{k_2} \times [\cC] \to \mathbb R$ and some $\ell^2$-regularization parameter $\Lambda\geq0$. Furthermore, $y=y(J)$ is a function of $J$.
\end{defn}

Let us understand this definition by way of example. Consider supervised classification: We are given i.i.d. samples, $\{(X^\ell,y^\ell)\}$, where $y^\ell$ are class labels, $y:[k]\to\cC$, and $X^\ell$ are features drawn from $\nu$-MM. Our goal is to learn a classifier $\hat y(X^\ell,\Theta)$.  
Regression tasks where we seek to learn a ground truth parameter $\btheta_*$, also fit in to this framework when the loss depends on the data $X$ through the pair $(\btheta^\top X,\btheta_*^\top X)$ by augmenting the parameter space by the ground truth parameter as a singleton, e.g., $\R^p\times\{\btheta_*\}$. (As this is a singleton, the SGD of course will not evolve in the ``augmented'' component.) Also in this setting, it is common to have centered features, and take $k=1$ and $\mu_1 =0$.
An example splitting of the parameter space would be when the learning is done with a multi-layer neural network, and $\btheta$ are the first-layer weights and $w$ are the hidden-layer weights.
In Section~\ref{sec:ballistic-universality-examples}, we present our results on such concrete examples.

\subsection{Universality of the ballistic phase}
We now turn to our main universality result for the ballistic phase of SGD run for linear order $d$ steps, in the high-dimensional asymptotic of $\delta = O(1/d)$ and $d\to\infty$. In particular, we suppose that $\lim_{d\to\infty} d \delta $ exists and is some $c_{\LR} \in [0,\infty)$. 
When $c_{\LR}>0$, this means that order $d$ many samples are being used in the runtime of the SGD, putting us in the so-called ``proportional asymptotic'' regime. The quantities $k,k_1,k_2,\cC,\Lambda,c_{\LR}$ are all fixed, meaning that the parameter dimension $p$ will also be proportional to~$d$. 

We will need an assumption on the projective model, $(L,\cP_X)$, that relates the rate of growth of the derivatives of $\psi$ to the number of moments of $\nu$. To this end, for $a\in [k_1],b\in [k_2]$, let $\partial_{1,a}$ denote derivatives of $\psi$ in the $a$th coordinate of its first argument and let $\partial_{2,b} \psi$ denote derivative in the $b$th coordinate of its second argument. Let $B_R(0)$ be the $\ell_2$ ball of radius $R$ about zero. 
For any $q,r\geq1$, define the function class $\cF_q$ to be those $f\in C^3(\R^r)$  with $\|f \|_{\mathcal F_q}<\infty$ where 
\[
\norm{f}_{\cF_q} =\inf\{K>0: \max\{|f|,\max_{a}|\partial_{a}f|,\max_{a,b}|\partial_{ab}f|,\max_{a,b,c}|\partial_{abc}f|\}(x)
 \leq K(1 + \norm{x}^q) \,\,\, \forall x\in\R^r\}\,. 
\]
(Note that this is only a quasi-norm.)
In words, this is the set of thrice continuously-differentiable functions whose partial derivatives up to order three are all of polynomial growth of order $q$ with uniform bound on the constants.

\begin{assumption}\label{assump:main}
    There exists $q \ge 1$ such that $\nu$ has $\max \{20q+4 , q^2+4\}$ finite moments and  for all $R>0$, 
    \begin{equation}\label{eq:psi-regularity}
    \sup_{w\in B_R(0)}\max_y \max_{a,b}\{\norm{\partial_{1,a}\psi(\cdot,w,y)}_{\cF_q}\vee\norm{\partial_{2,b} \psi(\cdot,w,y)}_
{\mathcal F_q}\}<\infty\,.
    \end{equation}
\end{assumption}

We note here that we did not work to optimize the $q$ dependence of the number of moments assumed on $\nu$, and do not expect it to be optimal.  We also note that without loss of generality, $\nu$ has variance one as $\sigma^2$ can be incorporated into the choice of $\psi$. 

The regularity on $\psi$ assumed in Assumption~\ref{assump:main} is sufficient to ensure an explicit ODE limit for all projective models with Gaussian mixture data. Namely, if we consider the family of summary statistics given by the Gram matrix 
\begin{equation}\label{eq:G-def}
    \boldsymbol{G}(\btheta,\bmu) = (\btheta,\bmu)^\top (\btheta ,\bmu)\,,
\end{equation}
together with the remaining parameters $w = (w_1,\ldots,w_{k_2})$, then under the evolution of SGD~\eqref{eq:SGD}, $\mathbf{u}_t^d$ which denotes the linear interpolation of $(\mathbf{G}(\btheta_{\lfloor t\delta^{-1}\rfloor},\bmu),w_{\lfloor t\delta^{-1}\rfloor})$ converges as $d\to\infty$ to the solution of an ODE $d\mathbf{u}_t = \mathbf{h}(\mathbf{u}_t)dt$ where $\mathbf{h}$ is as in \eqref{eq:effective-dynamics-explicit}. For a precise statement, see Theorem~\ref{thm:Gaussian-projective-has-effective-dynamics}.

We will also need an assumption on the initializations and mean vectors which we can admit. 
\begin{defn}\label{def:delocalized-set}
    The set ${\deloc}_{\zeta} \subset \mathbb R^d$ of \emph{coordinate-delocalized vectors} is the set defined as follows: 
    \begin{equation}
        {\deloc}_{\zeta} = \{ \theta \in \mathbb R^d: |\theta_i|\le d^{ - 1/2 + \zeta} \text{ for all $i$}\}\,.
    \end{equation}
     We write $\btheta \in {\deloc}_{\zeta}$ or $\Theta \in {\deloc}_{\zeta}$ if $\theta^a \in {\deloc}_{\zeta}$ for all $a\in [k_1]$. 
\end{defn}
In words, for $\zeta$ small, a vector being coordinate-delocalized says its typical coordinate size is roughly $d^{-1/2}$, up to some fluctuations. Observe that for any $\zeta>0$, an i.i.d.\ vector of order-$1$ norm, and with enough moments on its entries, is coordinate-delocalized with high probability. 
Our universality of ballistic limiting dynamics will hold if the initialization and mean vectors are coordinate-delocalized. As we will see in Section~\ref{subsubsec:bad-init-no-ballistic-univ}, this latter assumption is vital to the universality result. Now we can state our main universality result.

\begin{thm}[Universal ballistic limit]\label{thm:ballistic-universality}
    Fix any $\zeta<1/8$. Suppose $X$ is drawn from any $\nu$-MM with class means $\bmu \in {\deloc}_{\zeta}$ with $\max_a \|\mu^a\|= O(1)$, and noise distribution $\nu$ with $\E_{\nu}[Y_1]=0$ and $\E_{\nu}[Y_1^2]=1$. 
    Suppose that $(L,\cP_X)$ is a projective model satisfying Assumption~\ref{assump:main}. 
    
    Let $\Theta_\ell = (\btheta_\ell,w_\ell)$ be SGD initialized from $\Theta_0 \in {\deloc}_{\zeta}$ with step-size $\delta$ such that $d\delta \to c_{\LR}$. 
    Then $(\bu_t^d)_{t>0}\to(\bu_t)_{t>0}$ in $C[0,\infty)$ where $\bu_t$ solves the (explicit) ODE: 
    \begin{equation}\label{eq:ballistic-dynamics}
         d\mathbf{u}_t = \mathbf{h}(\mathbf{u}_t) dt\,.
    \end{equation}
    with initial data given by $\lim_{d\to\infty} \bu^d_0$. Here $\mathbf{h}$ is as in \eqref{eq:effective-dynamics-explicit} and does not depend on $\nu$.
\end{thm}

\subsubsection{Proof outline}\label{subsec:proof-outline}
Using the framework of~\cite{BGJ22}, we show that the limiting evolution of the summary statistics $\bu^d$ is given by \eqref{eq:ballistic-dynamics} for Gaussian projective models satisfying Assumption~\ref{assump:main}. While results of this type have been stated for specific tasks, we expect the general family of Gaussian projective models to be of independent interest. See Section~\ref{sec:Gaussian-projective-models}. 

The aim is then to show that the evolution of summary statistics under any other noise distribution stays within $o(1)$ of the Gaussian equivalent on all linear timescales. 
To this end, we leverage the structure of projective models as follows. By standard martingale concentration arguments, the evolution of summary statistics under SGD is governed by their drifts. These drifts are given by expectations of functions of the inner products $\btheta^\top X$. If these inner products satisfy a multivariate central limit theorem (CLT) under the noise distribution $\nu$ (together with some quantitative error rates), then we can treat them like the Gaussian case. The Lyapunov CLT suggests that this holds when the parameter vector $\btheta$ is coordinate-delocalized, and $Y$ has i.i.d. entries. We need a quantitative CLT of this type, which we derive in Section~\ref{sec:expectations-close}. 

Our work to show this is primarly  then to establish that for SGD with general $\nu$ noise distribution, if the initialization for $\btheta$ is coordinate-delocalized, then it remains coordinate-delocalized for all linear timescales with high probability. Specifically, we show that  $\|\btheta\|_\infty=O(d^{-1/2+\zeta})$ for $\zeta$ small by a Gronwall argument. This requires sharp quantitative control on the drift of $\langle \theta^a,e_i\rangle$. See Section~\ref{sec:parameter-stays-coordinate-delocalized} for this argument. 
The proof of Theorem~\ref{thm:ballistic-universality} is then concluded in Section~\ref{sec:main-ballistic-universality}.

\subsubsection{Importance of coordinate-delocalized initializations}\label{subsubsec:bad-init-no-ballistic-univ}
    It is natural to wonder if the coordinate-delocalized condition we require on the initialization is truly needed. When the feature data is isotropic Gaussian, any two initializations with the same limit law for the summary statistics admit the same effective dynamics regardless of how coordinate aligned they are. 
    By contrast, under other i.i.d.\ noise distributions we consider here, different initializations with the same limiting summary statistics can admit different ballistic dynamics in the limit. 
    We demonstrate this in the well-known phase retrieval problem in Proposition~\ref{prop:no-ballistic-universality-bad-init} below, where coordinate-aligned initializations break ballistic universality.

\subsubsection{Relation to the ``Deep Gaussian Equivalence Conjecture''}\label{subsec:DGEC}

The authors of \cite{Goldt-PRX-universality,Goldt-et-al-Gaussian-Equivalence-Conjecture} proposed the universality of the law of the output of the first-layer in teacher-student networks (see Section~\ref{sec:ballistic-universality-examples}). This proposal has received much attention in the learning theory literature.
Specifically, Conjecture~1 of~\cite{Goldt-et-al-Gaussian-Equivalence-Conjecture}---which they called the \emph{Deep Gaussian Equivalence Conjecture} (DGEC)---had two forms: (a) An active learning form, where it was predicted that after $td$ steps of online SGD for $t>0$, the inner product of the features with the first layer weights, i.e., $\btheta_{\lfloor td\rfloor }^\top X$, are asymptotically Gaussian; and (b) An empirical risk minimizer (ERM) form where the same quantity evaluated at the empirical risk minimizer is predicted to be asymptotically Gaussian. 
There has been significant progress on (b)~\cite{Goldt-et-al-Gaussian-Equivalence-Conjecture,HuLu,Dandi-ERM-universality,MontanariSaeed-ERM-universality}, from which one can deduce universality of the risk evaluated at the ERM.
On our way to the proof of Theorem~\ref{thm:ballistic-universality}, we establish the following precise version of the active learning form (a) of the DGEC. Let $\tau_R$ be the exit time of $\Theta_\ell$ from the $\ell_2$ ball of radius $R$ around~$0$. For any two probability measures, $\mu,\nu,$ let $d_{C^3_b}(\mu,\nu)=\sup_{\|f\|_{C^3_b} \leq 1}|\E_\mu f-\E_\nu f|$. 

\begin{thm}\label{thm:dgec}
    Fix any $\zeta <1/8$ and $c_{\LR}>0$. There exists $\varepsilon(\zeta,c_{\LR})>0$ such that the following holds for all $\bmu\in \mathcal D_\zeta$, and step-size $\delta = \frac{c_{\LR}}{d}$.  Let $\Theta_\ell = (\btheta_\ell,w_\ell)$ be SGD with $\nu$-MM data satisfying Assumption~\ref{assump:main}, and $\bar \Theta_\ell = (\bar \btheta_\ell,\bar w_\ell)$ be SGD with GMM data, both initialized from  $\Theta_0 = \bar \Theta_0 \in \mathcal D_\zeta$. 
    For any $R>0$, with probability $1-o(1)$ over the realization of SGD, as $d\to\infty$,
    \begin{align*}
        \max_{\ell \le (\varepsilon d\log d) \wedge \tau_R\wedge \bar \tau_R} d_{C^3_b}(\mathrm{Law}(\btheta_\ell^\top X) ,\mathrm{Law}( \btheta_\ell ^\top \bar X)) \to 0
    \end{align*}
    where $X, \bar X$ are independent data points from $\nu$-MM and GMM respectively. 
    \end{thm}

Proving this DGEC for all linear timescales of online SGD is the bulk of the work for establishing universality of the ballistic limits.
However, we emphasize two important caveats: Firstly, if the initialization or ground truth vectors are not sufficiently coordinate-delocalized, the DGEC is not true: see Proposition~\ref{prop:no-ballistic-universality-bad-init}. 
Secondly, the restriction to timescales less than $\tau_R\wedge\bar\tau_R$ is necessary: Due to the level of generality we consider, the norm of the parameters may blow-up in $\Omega(d)$ time. It is easy to see that stronger assumptions on the growth of $\psi$ at infinity, together with positive regularizer $\Lambda$, are enough to guarantee $\tau_R\wedge \bar \tau_R=\omega(d)$ with high probability.

\subsubsection{Universality of critical and subcritical scaling}
The ODE in~\eqref{eq:ballistic-dynamics} undergoes a phase transition in the step-size at a critical scale $\delta \asymp 1/d$, where if $\delta = o(1/d)$ then the summary statistics evolve as under gradient flow for the population loss, while at $\delta \asymp 1/d$, correction terms appear and the infinitesimal generator is second order.
In particular, our results show that the critical step-size scaling is universal. 
Note also that in the subcritical regime, $\delta  = o(1/d)$, the scaling limit of the summary statistics is universal for large $d$, even though the population loss is not.

\subsection{Examples}\label{sec:ballistic-universality-examples}
We now discuss two classes of examples which have received a tremendous amount of attention in recent years for Gaussian data for which our universality results carry over to non-Gaussian data.

\subsubsection*{Classification of mixture distributions}
Suppose that we are given features drawn from with a $k$-mean $\nu$-mixture model as in \eqref{eq:X-def}.
We will refer to the random variable $J\in [k]$ as the \emph{hidden label}. 
In addition, each feature, $X$, comes with a \emph{class label} $y \in [{k_1}]$, where $y$ is a deterministic function of the corresponding hidden class label. (We abuse notation and denote this function by $y(J)$.) In particular, there may be more than one mean corresponding to the same class, but the choice of mean dictates the class fully. We naturally identify the label $y$ with a one-hot vector $y\in \{0,1\}^{k_1}$.  We denote the distribution of the labeled data $(X,y)$ by $\mathcal P_X$. 

Given this data, we seek ${k_1}$ distinct one-vs-all hyperplane classifiers, whose normal vectors are encoded by $\btheta = (\theta^a)_{a\in{k_1}}$.\footnote{One can also include a ``bias term'' as in common practice by augmenting $\theta_a$ by an appropriate bias in the standard fashion, $\tilde\theta_a=(\theta_a,b_a)$.} We find these classifiers by optimizing the cross-entropy loss
\begin{equation}\label{eq:loss-function-1-layer}
    L(\Theta,(y,X)) = - \sum_{c\in [{k_1}]} y_c \theta^c\cdot X  + \log \sum_{c\in [{k_1}]} \exp( \theta^c \cdot X)\,,  \qquad\text{where} \quad \Theta= (\theta^a)_{a\in [{k_1}]}\,,\, \theta^a \in \mathbb R^d\,.
\end{equation}
This is clearly a projective model:
The loss function~\eqref{eq:loss-function-1-layer} is expressible as a function $\psi(\btheta^\top X, y)$ for a smooth function $\psi$  with at most linear growth at infinity, and Assumption~\ref{assump:main} holds with $q=1$.   When the noise distribution is Gaussian, the effective dynamics for the family of summary statistics $\mathbf{G}_{\lfloor t\delta^{-1}\rfloor} = (\btheta_{\lfloor t\delta^{-1}\rfloor},\bmu)^\top (\btheta_{\lfloor t\delta^{-1}\rfloor},\bmu)$ were derived in \cite[Theorem 5.7]{BGJH23} as a direct application of the results of \cite{BGJ22}. Applying Theorem~\ref{thm:ballistic-universality}, we arrive at the following universality result.

\begin{cor}\label{logistic-regression}
    Consider SGD for the logistic regression task~\eqref{eq:loss-function-1-layer} with $X$ drawn from the $\nu$-MM where the noise distribution $\nu$ has mean-zero, variance $\sigma^2$ and at least $24$ finite moments, and the mean-vectors are all in ${\deloc}_{\sfrac{1}{10}}$. If the SGD is initialized from $\Theta_0\sim \mathcal N(0,I_d/d)$, then the evolution of the summary statistics $(\mathbf{G}_{\lfloor t\delta^{-1}\rfloor})_{t\in [0,T]}$ converges as $d\to\infty$ to the same limiting ODE as it does under the $\mathcal N(0,\sigma^2)$ noise distribution. 
\end{cor}

When the class labels are such that different class means are not linearly separable, the minimizer of~\eqref{eq:loss-function-1-layer} will not result in a good classifier. A well-known example of this is XOR-type data distributions (see e.g.,~\cite{minsky1969introduction} for background on this type of problem as an early example of functions a single-layer network cannot learn) where $k=4$, the four means are $\pm \mu,\pm \nu$, and one class corresponds to hidden label $\pm \mu$, while the other corresponds to $\pm \nu$.

In such tasks where the means from different classes are not linearly separable, one needs a multi-layer neural network to express a good classifier.  Let us consider a simple two-layer architecture, with $k_1 = k_2= O(1)$ hidden neurons, activation function $g$ on the hidden neurons, and sigmoid activation at the output layer. That is, the loss function takes the form 
\begin{equation}\label{eq:loss-function-2-layer}
    L(\Theta,(X,y)) = -\sum_{c\in [{\cC}]} y_c w^c \cdot g(\btheta^c X) + \log \sum_{c\in [{\cC}]} \exp(w^c \cdot g(\btheta^c X))\,.
\end{equation}
For each $c\in [{\cC}]$ the parameter $\Theta^c = (\theta^c,w^c)$ generates the $c$'th one-vs-all classifier, where $\btheta^c\in \mathbb R^{d\times k_1}$ is the first layer weights $(\theta^{1,c},\ldots,\theta^{k_1,c})$, and $w^c\in \mathbb R^{k_1}$ is the second layer weights in the $c$'th one-vs-all classifier being trained. The activation $g$ is applied entry-wise. This is sometimes called a teacher-student network. 

It is easy to that this model is projective, and when $X$ is Gaussian mixture, its ballistic dynamics limits have been derived. Indeed, this was done in~\cite{refinetti2021classifying,BGJ22} in the special XOR case (see also the related setup of~\cite{Glasgow2023SGDFT}), and more generally in~\cite{BGJH25+} where the dynamics were studied in conjunction with the local loss landscape geometry. By verifying Assumption~\ref{assump:main}, we establish their universality.

\begin{cor}\label{cor:multi-layer-GMM}
    If $g\in\mathcal F_q$, $\nu$ has $\max \{20q +4 , q^2 +4\}$ moments, and the mean-vectors are in ${\deloc}_{\sfrac{1}{10}}$, then the summary statistics $\Theta = (\mathbf{G},w)$ evolved under SGD for~\eqref{eq:loss-function-2-layer} initialized from $\Theta^c \sim \mathcal N(0,I_d/d)$ for each $c$, and $w^c$ initialized arbitrarily, admits the same ODE limit as in case where the noise distribution is $\mathcal N(0,\sigma^2)$. 
\end{cor}

Corollary~\ref{cor:multi-layer-GMM} applies to popular activation functions such as the smoothed ReLU, the sigmoid, and the hyperbolic tangent.

\subsubsection*{Single and multi-index models}
Let $\Theta_*=(\theta_*^1,\ldots,\theta^k_*)$ be a $k$-tuple of fixed unit vectors on $\R^d$. Suppose that we are given a (non-linear) \emph{activation function}, $g:\R^k\to\R$,  and some \emph{feature vectors}, $(X^\ell)$, and responses of the form
\begin{equation}\label{eq:multi-index-regression-data}
g(\Theta_* ^\top X^\ell) = g(\theta_*^1 \cdot X^\ell,\ldots,\theta_*^k\cdot X^\ell)\,. 
\end{equation}
Our goal is to infer $\Theta_*$ by minimizing the $\ell^2$ loss 
over the parameter 
$\Theta = (\theta^1,\ldots,\theta^k) \in \mathbb R^{d\times k}$:
\begin{equation}\label{eq:multi-index-regression-loss}
    L(\Theta, X)= |g(\Theta ^\top X) - g(\Theta_*^\top X)|^2\,.
\end{equation}
Evidently, if $\Theta = \Theta_*$ then the loss is zero. To fit this problem into the framework of our paper, we take the number of classes to be one and the means $\bmu$ to be zero, as the features are typically centered, and augment the set of parameters to include both $\Theta$ and $\Theta_*$. By letting the parameter space be $\mathbb R^{d\times k} \times \{\Theta_*\}$, derivatives of the loss in the parameters are understood to only be in $\Theta$, so that SGD is only training $\Theta$ while the ground truth vectors $\Theta_*$ are fixed. 

The Gaussian case of the single and multi-index models has seen tremendous attention in recent years as a family of statistical tasks that exhibit different computational sample complexities for the performance of SGD, depending on a certain property of the link function called the \emph{information exponent}~\cite{pmlr-v75-dudeja18a,BGJ21} for the single index case, and the leap complexity~\cite{abbe2023sgd} in the multi-index case. See also concepts like the generative exponent~\cite{computational-statistical-Damian-et-al}, and ~\cite{DamianLeeSoltanolkotabi-22,simsek2024learninggaussianmultiindexmodels,mousavi-hosseini2023neural,oko2024neural,NEURIPS2023_02763667,NEURIPS2023_e21955c9} for a sampling of related work. The ballistic and diffusive limiting dynamics for the single-index model were computed following~\cite{BGJ22} in the recent paper~\cite{Rangriz-SGD-single-index} for the family of summary statistics; gradient flow on the population loss (which corresponds to the ballistic dynamics limit in the small step-size $c_{\LR} \to 0$ limit) for the multi-index models was studied in~\cite{Bietti-et-al-CPAM}. The family of summary statistics in these problems are $\mathbf{G}_{\lfloor t \delta^{-1}\rfloor} = (\Theta_{\lfloor t \delta^{-1}\rfloor},\Theta_*)^\top (\Theta_{\lfloor t \delta^{-1}\rfloor},\Theta_*)$. 

\begin{cor}\label{cor:multi-index-regression}
    Consider the multi-index model of~\eqref{eq:multi-index-regression-loss} with mean zero, variance one i.i.d.\ features $X_i \sim \nu$ with coordinate-delocalized ground truth vectors $\Theta_* \in {\deloc}_{\sfrac{1}{10}}$. Suppose that the link function $g \in \mathcal F_q$ and $\nu$ has $\max\{40q +4,4q^2+4\}$ finite moments. If the SGD is initialized from $\Theta_0 \sim \mathcal N(0,I_d/d)$, the summary statistics $(\mathbf{G}_{\lfloor t \delta^{-1}\rfloor})_{t\in [0,T]}$ admit the same limit as in the Gaussian case where the features are i.i.d.\ $\mathcal N(0,1)$. 
\end{cor}

An extension of the multi-index model task is in the situation where even the link function $g$ is unknown to the statistician, and is to be learned using a multi-layer neural network. We describe the formulation found for example in the paper~\cite{Bietti-et-al-CPAM}, which studied the effective dynamics for this problem.  
In this case, the goal is to learn the function $h_*:X \mapsto  g(\Theta_*^\top X)$ by a two-layer neural network with bounded width in its hidden layer. Namely, the loss function will now be given by 
\begin{equation*}
    L(\Theta;(h_*(X),X)) = \big|h_{\NN}(X) - h_*(X)\big|^2 \qquad\text{where} \qquad h_{\NN}(X) = w \cdot \sigma(\btheta^\top X)
\end{equation*}
where the parameters $\Theta = (\btheta,w)$ are the first and second layer weights of a fixed width neural network, $\btheta \in \mathbb R^{K\times d}$ and $w\in \mathbb R^K$, and where $\sigma$ is the \emph{activation function} applied entry-wise. 
This context again clearly to fits into the framework of our paper.

\subsection{Further extensions}
In this paper, we have focused on a simple setting to present our ideas, which captures many models of interest. That said, the ideas developed in this paper can be readily applied to much broader classes of models after minor modifications.  Before turning to our examples of non-universal fluctuations, we pause here to discuss a few such natural extensions. 

The arguments of the paper  directly adapt to handle cases where the $\ell^2$-penalty is non-isotropic. For example, the strength can depend on the choice of which $\theta^a$ vector i.e., as $\sum_{a\in [k_1]} \lambda_a \|\theta^a\|$, or more generally where the loss depends on the full Gram matrix of $\Theta$ as  $\psi(\btheta^\top X,\btheta^\top \btheta, w; y)$. 

Similarly, our proof naturally generalizes to allow for different noise distributions, $\nu_a$, for different classes, and even $(Y_i)_{i=1}^d$ that are independent with the same mean and variance but not identically distributed. That said, the product structure of the internal noise distribution is essential to our arguments and it is an important question to relax this assumption in a non-trivial way.

Furthermore, in regression settings, it is natural to assume that the variable $y$ in \eqref{eq:projective-def} is the response and thus real-valued rather than discrete and there is some ``ground truth'' vector $\Theta_*$ on which $y$ depends in a projective fashion, i.e., $y= y(\Theta_*^\top X)$.
This would provide an equivalent way to e.g., encompass the  multi-index models, besides the trick we utilized of augmenting the parameter space by taking its cross product with the point $\{\Theta_*\}$. 
We also note that one may wish to include an additional, independent source of randomness, e.g., some authors include additive noise $\varepsilon_\ell$ in~\eqref{eq:multi-index-regression-data} which is independent of $\cP_X$.
It is clear from the proofs that one can incorporate an extra real-valued random variable into $\psi$ and if it has sufficiently many moments, the proofs go through with minor modifications, always using independence of $\varepsilon$ from $X$ to isolate it.

\subsection{Non-universal fluctuations}

One may expect that, along the lines of universality of Donsker's invariance principle or related functional central limit theorems, the fluctuations of the summary statistics about their ODE limits should obey an SDE that also is independent of~$\nu$. Surprisingly, we find this is not the case, and find a non-universality for the SDE fluctuations about fixed points of the universal ballistic dynamics. This is a  more subtle non-universality than the simple example of coordinate-aligned initialization breaking the ballistic universality described in Section~\ref{subsubsec:bad-init-no-ballistic-univ}, and holds for Gaussian, or other optimally coordinate-delocalized initializations.  

To be more precise, we recall that in the Gaussian case,~\cite{BGJ22} rescaled the summary statistics $\mathbf{u}$ about fixed points $\mathbf{u}_*$ of the ODE~\eqref{eq:effective-dynamics} as $\widetilde {\mathbf{u}} = \sqrt{d} (\mathbf{u} - \mathbf{u}_*)$. In many of the examples of Section~\ref{sec:ballistic-universality-examples}, the evolution of those rescaled summary statistics, on linear timescales, was known to evolve as an autonomous SDE. In Theorem~\ref{thm:Gaussian-projective-diffusive-limit}, we establish that $\widetilde{\mathbf{u}}(\Theta_\ell)$ weakly converges to an autonomous SDE for all Gaussian projective models satisfying~\eqref{eq:psi-regularity} with $\psi\in C^5$. 

We find that even for coordinate-delocalized initializations and ground truth parameters, and sub-Gaussian noise distribution $\nu$, the SDE limit is not universal. The counter-example is simple enough to describe here. Recall the single-index models from Section~\ref{sec:ballistic-universality-examples}, and consider the link function 
\begin{equation}\label{eq:2nd-plus-3rd-Hermites}
    g(x) =  \mathsf{He}_3(x) + \mathsf{He}_2(x)
\end{equation}
where $\mathsf{He}_k(x)$ denotes the $k$'th Hermite polynomial. This task is well-known to have information exponent two in the sense of \cite{BGJ21}. In our context, that means that under standard Gaussian data distribution, the SGD for this single-index model admit an ODE limit~\cite{Rangriz-SGD-single-index} with summary statistics $(m,R)$ where $m(\theta) = \langle \theta,\theta_*\rangle$ and $R = \|\theta\|_2^2$, and that ODE has an unstable fixed point at $m_* =0$ and $R_*(c_{\LR})>0$. A uniform-at-random initialization places $\theta_0$ at distance $O(1/\sqrt{d})$ of this uninformative fixed point, and the rescaled dynamics about the fixed point converge to an Ornstein--Uhlenbeck process that is mean-repellent. This indicates the fact that $\Theta(n \log n)$ samples are needed to learn $\theta_*$ in this problem via online SGD. The following lemma establishes that under non-Gaussian data distribution, the SDE around the (universal) ballistic fixed point are different from the Gaussian ones. The fact that the information exponent depends on the data distribution was investigated in detail in~\cite{Bruna-beyond-Gaussian-data} (see also~\cite{pmlr-v125-yehudai20a,Wu2022LearningAS} for earlier works in these directions). 

\begin{thm}\label{thm:no-diffusive-universality}
    Suppose $\theta_* = \rho d^{-1/2}\mathbf 1$ and suppose $\langle \theta_0,\theta_*\rangle =O(d^{-1/2})$ with $\|\theta_0\|_2^2 = R_* + O(d^{-1/2})$. Consider SGD with the single-index model of~\eqref{eq:2nd-plus-3rd-Hermites} initialized from $\theta_0\in {\deloc}_{\sfrac{1}{10}}$ with step-size $\delta=c_{\LR}/d$ and with feature distribution $\nu$ having mean-zero, variance-one, and non-zero third moment. Then for $\rho>0$ sufficiently small, (subsequential) limit points of the interpolated summary statistic trajectories $\widetilde{\mathbf{u}}^d_{t}= \widetilde {\mathbf{u}}(\theta_{\lfloor t\delta^{-1}\rfloor})$ do not equal to the limit under the i.i.d.\ $\mathcal N(0,1)$ feature distribution.  
\end{thm}

Given the simplicity of the counterexample to universality of fluctuations of summary statistic evolutions, we expect there to generically be non-universality of the SDE limits of summary statistics zoomed in about their fixed points.  

Indeed, the idea of the diffusive non-universality is that there is a finite-$d$ correction to the drift function coming from the Berry--Esseen correction to the central limit theorem. This  finite-$d$ correction to the dynamical system as a function of the summary statistics $\mathbf{u}$ is exactly of order $1/\sqrt{d}$ (sharp for Berry--Esseen). This vanishes in the ballistic limit, but locally around a zero of the drift function of the ballistic dynamics, in the rescaled summary statistics this is exactly amplified by the right amount to form an extra drift term. In particular, exactly at $\mathbf{u}_*$, the non-Gaussian dynamics has a drift of $1/\sqrt{d}$ which becomes order $1$ when blown up diffusively, while the Gaussian dynamics' drift is a $d$-independent function of the summary statistics and therefore has drift exactly $0$ at $\mathbf{u}_*$.  This argument can be found in Section~\ref{sec:non-universal-diffusive-limit}.

\subsection*{Acknowledgments}
The authors thank  Sebastian Goldt and Bruno Loureiro for pointing them to Conjecture~1 of~\cite{Goldt-et-al-Gaussian-Equivalence-Conjecture} and related literature.
R.G.\ thanks Subhabrata Sen for interesting discussions related to this problem. The research of R.G.\ is supported in part by NSF CAREER grant 2440509 and NSF DMS grant 2246780.  A.J.\ acknowledges the support of the Natural Sciences and Engineering Research Council of Canada (NSERC) and the Canada Research Chairs program [RGPIN-2020-04597, DGECR-2020-00199,CRC-2022-00142]. Cette recherche a \'et\'e enterprise gr\^ace, en partie, au 
soutien financier du Conseil de Recherches en Sciences Naturelles et en G\'enie du Canada (CRSNG),  et du Programme des chaires de recherche du Canada.

\section{Gaussian Projective models}\label{sec:Gaussian-projective-models}

In this section, we  show that when the data distribution is a Gaussian mixture model, the projective models of Definition~\ref{def:projective-model} under Assumption~\ref{assump:main} equipped with specific set of summary statistics, $\bu=(\mathbf{G},w)$, readily fall into the effective dynamics framework of~\cite{BGJ22}, with explicit formulas for the effective  dynamics (both in the ballistic and diffusive regimes). These notions have been underlying much recent work on high-dimensional Gaussian tasks that admit a low-dimensional structure that can be exploited for analysis. This section formalize the general requirements on Gaussian projective models to admit autonomous family of $O(1)$-many summary statistics.

\begin{defn}
    A projective model is called a \emph{Gaussian projective model} on $\R^d$ if the internal noise distribution is Gaussian, $\nu= \cN(0,\sigma^2)$.
\end{defn}

We will establish that Gaussian projective models satisfy the conditions of~\cite{BGJ22}, meaning they have a finite family of summary statistics that admit an autonomous high-dimensional limit.  
Observe that $\bfG$ from~\eqref{eq:G-def} has a natural block structure of the form
\[
\bfG = 
\begin{bmatrix}
\bfG^{\theta\theta} & \bfG^{\theta\mu}\\
(\bfG^{\theta\mu})^\top & \bfG^{\mu\mu}
\end{bmatrix},
\]
where the blocks are the relevant inner products.  

To precisely state these dynamics, we need to introduce further quantities. 
First, for each $a\in[k]$, we define the Gaussian vector $\cZ^{(a)}\in\R^{k_1+k}$ with mean and covariance
\begin{equation}
    \cZ^{(a)} \sim \cN([\bfG^{\theta\mu}_{\cdot a},\bfG^{\mu\mu}_{\cdot a}],\bfG)\,.
\end{equation}
Note that if $X$ is drawn from the mixture component with mean $\mu^a$, then $X^\top (\btheta,\bmu)$ is equal in distribution to $\mathcal Z^{(a)}$. This random vector has the natural block structure $\cZ^{(a)}=[\cZ^{(a),\theta},\cZ^{(a),\mu}].$

The drift for the ballistic dynamics will be given by the following Gaussian expectations:         
\begin{align}\label{eq:effective-dynamics-explicit}
\mathbf{h}_u =  
\begin{cases}
    -\sum_a p_a\E[\cZ^{(a),\mu}_c \partial_{1,b} \psi  ] - 2 \Lambda \bfG^{\theta\mu}_{bc}  & u = \mathbf{G}_{bc}^{\theta\mu} \\     
    -\sum_a p_a\E\big[(\cZ^{(a),\theta}_c\partial_{1,b}+\cZ^{(a),\theta}_b\partial_{1,c}) \psi+c_{\LR} \partial_{1,b}\psi\partial_{1,c}\psi\big] - 2 \Lambda \bfG^{\theta\theta}_{bc}   & u  = \mathbf{G}_{bc}^{\theta\theta} \\ 
    - \sum_a p_a\mathbb E[\partial_{2,b} \psi ]  & u = w_b
\end{cases}
\end{align}
where in the above expectations $\partial_{1,b}\psi,\partial_{1,c}\psi$ and $\partial_{2,b}\psi$ are evaluated at $(\cZ^{(a),\theta}, w, y)$

\begin{thm}\label{thm:Gaussian-projective-has-effective-dynamics}
    A projective Gaussian model  with $\psi$ satisfying~\eqref{eq:psi-regularity}, with learning rate $\delta$ such that $d\delta \to c_{\LR}$, admits the following effective dynamics: 
    \begin{align}\label{eq:Gaussian-effective-dynamics}
       d \mathbf u_t  =  \mathbf{h}(\mathbf{u}_t)dt
    \end{align}
    where $\mathbf{h}$ is given by~\eqref{eq:effective-dynamics-explicit}. 
    By that, we mean that if the law of $ \mathbf{u}^d(\Theta_0)$ converges weakly to some $\pi$, then  $ (\mathbf u^d_t)_{t}\to (\bu_t)_{t}$
    weakly 
    for $u_t$ solving the above ODE initialized from $\pi$.
\end{thm}

    Furthermore, if we consider the rescaled summary statistics then we get diffusive limits of the following form. Let $\mathbf{u}_*$ be a fixed point of the ODE system of~\eqref{eq:Gaussian-effective-dynamics} and define $\widetilde{\mathbf{u}} = \sqrt{d}(\mathbf{u} - \mathbf{u}_*)$. For this rescaled process, we introduce the following functions which will be its effective drift and volatility. First the effective drift will be given by, for each summary statistic $\tilde u = \sqrt{d} (u - u_*)$, 
        \begin{align}\label{eq:Gaussian-SDE-drift}
    \widetilde{\mathbf{h}}_{\tilde u} (\widetilde{\mathbf{u}}) = \langle \nabla \mathbf{h}_{u}(\mathbf{u}_*), \widetilde{\mathbf{u}}\rangle\qquad \text{for $\mathbf{h}_u$ from~\eqref{eq:effective-dynamics-explicit}}\,.    
    \end{align}
    The fact that $\nabla \mathbf{h}_{\tilde u}(\mathbf{u}_*)$ exists and has an exact formula can be seen by a Gaussian integration-by-parts argument. (See Lemma~\ref{lem:locally-Lipschitz} which is a formula for the derivative of the Gaussian expectation of a $C^2$ function, in its mean and covariance matrix.)
    The effective volatility matrix is constant and given by $\mathbf{\Sigma}(\widetilde{\mathbf{u}}) \equiv \mathbf{\Sigma}$ for
    \begin{align}\label{eq:volatility-matrix}    \mathbf{\Sigma}_{\tilde{\mathsf{u}}\tilde{\mathsf{u}}'}  = c_{\LR} \text{Cov}\big( S_{\mathsf{u}}, S_{\mathsf{u'}}) (\mathbf{u}_*)
    \end{align}
    where 
\begin{align}\label{eq:volatility-matrix-II} 
    S_{\mathsf{u}} = \begin{cases} 
          \cZ^{(A),\mu}_b\partial_{1a}\psi   & \mathsf{u} = \langle \theta^a,\mu^b\rangle  \\ 
            \cZ^{(A),\theta}_a\partial_{1b} \psi  +  \cZ^{(A),\theta}_b\partial_{1a} \psi   &  \mathsf{u} = \langle \theta^a,\theta^b\rangle \\ 
          \partial_{2a}\psi &  \mathsf{u} = w_a
        \end{cases}
    \end{align}
    where the derivatives of $\psi$ are evaluated at $(\cZ^{(A),\theta},w,A)$  
and where $A\sim \Cat((p_a)_{a\in [k]})$.

    \begin{thm}\label{thm:Gaussian-projective-diffusive-limit}
        Suppose we have a Gaussian projective model with $\psi\in C^5$ satisfying~\eqref{eq:psi-regularity} and learning rate $\delta$ such that $d\delta \to c_{\LR}$, then if $\lim_{d\to\infty} \widetilde{\mathbf{u}}(\Theta_0)$ exists and is some $\widetilde{\pi}$, then the process $(\widetilde{\mathbf{u}}(\Theta_{\lfloor t \delta^{-1}\rfloor}))_{t}$ (linearly interpolated) converges as $d\to\infty$ to the solution of the SDE
        \begin{equation*}
            d\widetilde{\mathbf{u}}_t = \widetilde{\mathbf{h}}(\widetilde{\mathbf{u}}_t)dt + \widetilde{\mathbf{\Sigma}}\cdot d\widetilde{\mathbf{B}}_t
        \end{equation*}
        where $d\widetilde{\mathbf{B}}_t$ is standard Brownian motion in dimension of $\widetilde{\mathbf{u}}$, and the drift function  $\widetilde{\mathbf{h}}$ and volatility  $\widetilde{\mathbf{\Sigma}}$ are given explicitly as Gaussian integrals of derivatives of $\psi$ in~\eqref{eq:Gaussian-SDE-drift} and~\eqref{eq:volatility-matrix} respectively. 
    \end{thm}

    Observe in particular that Gaussian projective models, when rescaled around a fixed point of their ballistic dynamics, have a constant volatility matrix in the limiting SDE.

\subsection{Recalling conditions for limiting effective dynamics}
The discussion below recalls the main result of~\cite{BGJ22}. For this subsection, we will follow the notation of that paper. In particular, $n$ is a dummy parameter $n\to\infty$ and not necessarily the sample size.

Suppose that we are given a sequence of data $X_1,X_2,\ldots$ taking values in $\cY_n\subseteq \R^{d_n}$ with law $P_n\in\cM_1(\R^{d_n})$, a sequence of loss functions $L_n:\cX_n\times \cY_n\to\R$
 where here $\cX_n\subseteq \R^{p_n}$, and we are interested in online SGD with learning rate $\delta_n$.
Suppose that we are given a sequence of functions
$\mathbf{u}_n\in C^{1}(\R^{p_n};\R^{q})$  for some fixed $q$
where $\mathbf{u}_{n}(x)=(u_{1}^n(x),\ldots,u_{q}^n(x))$.
Our goal is to understand the evolution of $\mathbf{u}_n(\Theta_\ell)$. (To match the notation with our setting, we may take $n=d$, $d_n=d$, $\delta_n$ such that $d\delta_n \to c_{\LR}$, and  $p_n= d\times k_1+k_2$.)

In what follows, let $H_n(\Theta,Y)={L}_n(\Theta,Y)-\Phi_n(\Theta)$, where $\Phi_n(\Theta)  = \mathbb E [ {L}_n(\Theta,Y)]$ and let 
$V_n(\Theta)=\E_{Y}\left[\nabla H_n(\Theta,Y)\tensor\nabla H_n(\Theta,Y)\right]$
denote the covariance matrix for $\nabla H_n$ at~$\Theta$.

In order to develop a theory for the high-dimensional limiting trajectories of the functions $\mathbf{u}_n$, which we will call summary statistics following~\cite{BGJ22}, we need to assume: 
\begin{enumerate}
    \item A certain amount of regularity of moments of these functions and their derivatives, which will be relative to the step size $\delta_n$, and is called \emph{$\delta_n$-localizability}; 
    \item That in the dimension to infinity limit, the drift and volatility of the evolution of $\mathbf{u}_n$ are asymptotically expressible as functions of $\mathbf{u}_n$ themselves, rather than needing the entire vector in parameter space. This is called \emph{asymptotic closability} of the function family. 
\end{enumerate}
We now give the precise form of these two definitions before moving on to state the general theorem of~\cite{BGJ22}, which we will apply to the $k$-GMM classification task.  

\begin{defn}\label{defn:localizable}
A triple $(\mathbf{u}_n,L_n,P_n)$ is \textbf{$\delta_n$-localizable}
if for every $R>0$ there is  a constant $C_R$ (independent of $n$)
such that 
\begin{enumerate}
\item $\max_{i} \sup_{\Theta\in\mathbf{u}_{n}^{-1}(B_R(0))}\norm{\nabla^{2}u_{i}^n}_{\op}\leq C_R\cdot\delta_n^{-1/2}$,  and $\max_{i} \sup_{\Theta\in\mathbf{u}_{n}^{-1}(B_R(0))}\norm{\nabla^{3}u_{i}^n}_{\op} \leq C_R$;

\item $\sup_{\Theta\in\mathbf{u}_{n}^{-1}(B_R(0))}\|\nabla\Phi\|\le C_R$, and 
$\sup_{\Theta\in\mathbf{u}_{n}^{-1}(B_R(0))}\mathbb{E}[\|\nabla H\|^{8}]\le C_R\delta_{n}^{-{ 4}}$;

\item $\max_{i}\sup_{\Theta\in\mathbf{u}_{n}^{-1}(B_R(0))}\E[\langle \nabla H,\nabla u_{i}^n\rangle ^{4}]\leq C_R\delta_{n}^{-2}$, and \hfill 

$\max_{i} \sup_{\Theta\in \mathbf u_n^{-1}(B_R(0))} \mathbb E[\langle \nabla^2 u_i^n, \nabla H\otimes \nabla H - V\rangle^2] = o(\delta_n^{-{ 3}})$.
\end{enumerate}
\end{defn}

We add a specialization of the above definition when there are stronger bounds on some of the quantities that ensure that the scaling limit is an ODE.

\begin{defn}\label{defn:localizable-on-good}
A triple $(\mathbf{u}_n,L_n,P_n)$ is \textbf{strongly $\delta_n$-localizable on $A\subset \mathbb R^p$} if for every $R>0$ constants $C_R$ (independent of $n$)
such that 
\begin{enumerate}
\item $\max_{i} \sup_{\Theta\in\mathbf{u}_{n}^{-1}(B_R(0))\cap A}\norm{\nabla^{2}u_{i}^n}_{\op}\leq C_R$,  and $\max_{i} \sup_{\Theta\in\mathbf{u}_{n}^{-1}(B_R(0))\cap A}\norm{\nabla^{3}u_{i}^n}_{\op}\leq C_R$;

\item $\sup_{\Theta\in\mathbf{u}_{n}^{-1}(B_R(0))\cap A}\|\nabla\Phi\|\le C_R$, and 
$\sup_{\Theta\in\mathbf{u}_{n}^{-1}(B_R(0))\cap A}\mathbb{E}[\|\nabla H\|^{8}]\le C_R \delta_{n}^{-{ 4}}$;

\item $\max_{i}\sup_{\Theta\in\mathbf{u}_{n}^{-1}(B_R(0))\cap A}\E[\langle \nabla H,\nabla u_{i}^n\rangle ^{4}]\leq C_R$, and 

$\max_{i} \sup_{\Theta\in \mathbf u_n^{-1}(B_R(0))\cap A} \mathbb E[\langle \nabla^2 u_i^n, \nabla H\otimes \nabla H - V\rangle^2] = o(\delta_n^{-{ 3}})$.
\end{enumerate}
If this holds for $A=\R^p$, then we will omit $A$ and simply say that the triple is strongly $\delta_n$-localizable.
\end{defn}

Now define the following first and second-order differential operators, 
\begin{equation}\label{eq:A-L-operators}
\cA_n = \langle \nabla \Phi, \nabla\rangle\cdot\,, \qquad \mbox{and} \qquad \cL_n = \frac{1}{2} \langle V, \nabla^2\rangle\,.
\end{equation}
Let $J_n$ denote the Jacobian matrix $\nabla \mathbf{u}_n$.

\begin{defn}\label{defn:asympotically-closable}
A family of summary statistics $(\mathbf{u}_n)$ are \textbf{asymptotically closable} for step-size $\delta_n$ if $(\mathbf{u}_n, L_n, P_n)$ are $\delta_n$-localizable with localizing sequence $(E_R)_R$, and furthermore there exist locally Lipschitz functions $\mathbf{h}:\mathbb R^q\to\mathbb R^q$ 
and $\mathbf{\Sigma}:\R^{q}\to \mathbb R^{q\times q}$, such that
\begin{align}
\sup_{\Theta\in \mathbf{u}_n^{-1}(E_R)} \big\| \big( - \cA_n + \delta_n \mathcal L_n\big) \mathbf{u}_n(\Theta) - \mathbf h(\mathbf{u}_n(\Theta))\big\| & \to 0\,,  \label{eq:eff-drift} \\
\sup_{\Theta\in\mathbf{u}_{n}^{-1}(E_R)}\|\delta_{n}J_{n} V J_{n}^{T}-\mathbf{\Sigma}(\mathbf{u}_{n}(\Theta))\| & \to0\,. \label{eq:diffusion-matrix}
\end{align}
In this case we call $\mathbf{h}$ the \emph{effective drift}, and $\mathbf{\Sigma}$ the \emph{effective volatility}.
\end{defn}

For a function $f$ and measure $\pi$ we let $f_{*}\pi$
denote the push-forward of $\pi$. The main result of~\cite{BGJ22} was the following limit theorem for SGD trajectories as $n\to\infty$ .

\begin{thm}[{\cite[Theorem 2.2]{BGJ22}}]\label{thm:main-BGJ22}
Let $(\Theta_{\ell}^{\delta_{n}})_{\ell}$ be stochastic gradient descent initialized from $\Theta_{0}\sim\mu_{n}$
for $\mu_{n}\in \cM_{1}(\mathbb{R}^{p_{n}})$ with learning
rate $\delta_{n}$ for the loss $L_{n}(\cdot,\cdot)$ and data distribution
$P_{n}$. For a family of summary statistics $\mathbf{u}_n = (u_i^n)_{i=1}^q$, let $(\mathbf{u}^n_t)_t$ be the linear interpolation of $(\mathbf{u}_n(\Theta_{\lfloor t\delta_{n}^{-1}\rfloor}^{\delta_{n}}))_{t}$. 

Suppose that $\mathbf{u}_n$ are asymptotically closable with learning rate $\delta_n$, effective drift $\mathbf{h}$, and effective volatility $\mathbf{\Sigma}$, and that the pushforward of the initial data has $(\mathbf{u}_n)_*\mu_n \to \pi$ weakly for some $\pi\in \cM_1(\mathbb R^q)$. Then $(\mathbf{u}_{n}(t))_{t}\to(\mathbf{u}_t)_t$ weakly as $n\to\infty$, where
$\mathbf{u}_t$ solves
\begin{equation}\label{eq:effective-dynamics}
d\mathbf{u}_{t}= \mathbf{h}(\mathbf{u}_{t}) dt+\sqrt{\mathbf{\Sigma}(\mathbf{u}_{t})}d\mathbf{B}_{t}\,.
\end{equation}
initialized from $\pi$, where $\mathbf{B}_{t}$ is a standard Brownian
motion in $\mathbb{R}^{q}$. 
\end{thm}

In what follows, we drop the $n=d$ subscripts, leaving the dependence implicit.

\subsection{A Gaussian regularity lemma}
We will need the following standard estimate on the regularity of the expectation of a function of a Gaussian random variable, in its covariance matrix. 

Let $\sP_k$ be the space of $k\times k$ positive definite matrices.
Recall the following integration-by-parts formula sometimes called the ``second-order Stein's lemma'': If $f$ is $C^2$ with derivatives of polynomial growth and $W\sim\cN(0,A)$ for $A>0$ then
\begin{equation}\label{eq:ses}
    \E [\nabla^2 f(X)]=\E[f(X)[A^{-1}XX^TA^{-1}-A^{-1}]]\,.
\end{equation}

\begin{lem}\label{lem:locally-Lipschitz}
    Let $f\in C^2$ with derivatives of polynomial growth. For $X\sim \cN(\mu,A)$ consider the map $F:\R^k\times\sP_k\to \R $  given by
    \[
    F(\mu, A) = \E[f(X)]\,.
    \]
     Then this map is $C^1$ and 
    \[
    \nabla_\mu F=\E [\nabla f(X)] \qquad\qquad \nabla_A F = \frac{1}{2}\E[\nabla^2 f(X)].
    \]
    In particular, it is locally Lipschitz on $\R^k \times \sP_k$.
\end{lem}
\begin{proof}
The derivative in $\mu$ is clear. We focus on the derivative in $A$. Suppose first that $A>0$. Then this map is clearly differentiable. We compute the derivative at $A$ in the direction of $B$,  by standard matrix-calculus identities applied to the log-likelihood for $X$, that is, $\log p_X(x)$. In this case we have that
\[
\langle \nabla_A F,B\rangle=\frac{1}{2}\E[f(X)(\langle X,\Ainv B\Ainv X\rangle-\tr(\Ainv B)]
=\frac{1}{2}\tr\Big[B\E[f(X)(\Ainv XX^\top\Ainv-\Ainv) \Big].
\]
Applying \eqref{eq:ses} above, 
\[
\langle\nabla_A F,B\rangle = \frac{1}{2}\E\tr(B \nabla^2f(X))\,,
\]
which yields the desired identity.
Now observe that this expression is continuous on all of $\R^k\times\sP_k$, thus by a standard continuous extension argument, $F(\mu,A)$ is differentiable up to the boundary of $\R^k\times\sP_k$ as well, with tangential derivatives suitably defined.
\end{proof}

\subsection{Effective dynamics for Gaussian projective models}

We can now prove Theorem~\ref{thm:Gaussian-projective-has-effective-dynamics} by verifying the conditions of Theorem~\ref{thm:main-BGJ22} and matching the ODE of~\eqref{eq:effective-dynamics} to the claimed effective dynamics limit. We begin with the following intermediate lemma. 
\begin{lem}\label{lem:Gaussian-strongly-delta-localizable}
    Let $(L,\cP_X)$ be a sequence of projective Gaussian models on $\R^d$ satisfying~\eqref{eq:psi-regularity} for some $q\ge 1$ and let $\mathbf{u} = (\mathbf{G},w)$. The triple $(\bu,L,\cP_X)$ at learning rate $\delta$ such that $d\delta \to c_{\LR}$ is strongly $\delta$-localizable. 
\end{lem}

\begin{proof}
        We take variance $\sigma^2 = 1$ without loss of generality. Other variances can be captured by modifying $\psi$ and $\Lambda$.

    \paragraph{\emph{Item 1}} The summary statistics of the form $\langle \theta^a,\mu^b\rangle$ and $w = (w_1,\ldots,w_{k_2})$ are linear in $(\btheta, w)$, and so item 1 holds for them trivially. For summary statistics $\langle \theta^a, \theta^b\rangle$, the third derivative tensor is $0$, while the Hessian is an identity matrix in a sub-block and zero elsewhere; this has operator norm~$1$.

    \paragraph{\emph{Item 2}} The population loss $\Phi( \Theta) = \mathbb E[\psi(\btheta^\top X, w, y)]+ \Lambda \|\btheta\|_2^2 $ has 
    \begin{equation*}
        \|\nabla \Phi\|^2 = \sum_{a= 1}^{k_1} \|\nabla_{\theta^a} \Phi\|^2 + \sum_{b=1}^{{k_2}} |\partial_{w^b} \Phi|^2 \,.
    \end{equation*}
    Now  note that 
    \begin{equation*}
        \nabla_{\theta^a} \Phi = \mathbb E \big[(\partial_{1,a} \psi)  \, (\mu^J + Z)\big]   +  2 \Lambda \theta^a
    \end{equation*}
    where the expectation is over $J\sim \Cat(p_i)_{i=1}^k$ and $Z \sim \cN(0,I_d)$. Since $\|\mu^b\|= O(1)$ for each $b\in [k]$ we get 
    \begin{equation*}
        \|\nabla_{\theta^a}\Phi\|_2^2 \lesssim \Big(k_1 \mathbb E[ | (\partial_{1,a} \psi) |^2 ] + \|\mathbb E[(\partial_{1,a} \psi) (Z_\parallel + Z_\perp)]\|_2^2 \Big) + 4 \Lambda^2 \|\theta^a\|_2^2
     \end{equation*}
     where $Z_{\parallel}$ is the projection of $Z$ into $\text{Span}(\theta^1,\ldots,\theta^{k_1})$ and $Z_{\perp}$ is independent of $Z_{\parallel}$, using Gaussianity of $Z$. Since $\text{Span}(\theta^1,\ldots,\theta^{k_1})$ is ${k_1}$-dimensional, one has $\mathbb E[Z_\parallel^q] =O(1)$ for all $q\ge 1$. At the same time, $Z_\perp $ is independent of $\btheta^\top X$ and thus of $\partial_{1,a}\psi$. As a result, we get 
     \begin{equation*}
         \|\nabla_{\theta^a} \Phi \|_2^2 \lesssim \mathbb E[(\partial_{1,a} \psi)^2+(\partial_{1,a} \psi)^4+\norm{Z_\parallel}^4]+4\Lambda^2\norm{\theta^a}^2_2 \,.
     \end{equation*}
    For the derivatives in the second layer $w= (w_b)_{b\in [{k_1}]}$, we get
    \begin{equation*}
        |\partial_{w^b} \Phi|^2 \le \mathbb E[(\partial_{2,b} \psi)^2]\,.
    \end{equation*}
    The term $4\Lambda^2 \|\btheta\|^2$ is evidently bounded by $4\Lambda^2 R^2$ for $\btheta\in B_R(0)$. 
    Part (1) of item (2) therefore follows since for every $p,q\geq 1$, $R>0$, there exists $K_{R,p,q}$ such that
    \begin{align}\label{eq:main-assumption-on-psi}
        \sup_{(\btheta,w)\in B_R(0)} \max_{i\in [{k_1}],j\in [k_2]} \mathbb E[(\partial_{1,i} \psi)^p] \vee   \mathbb E[(\partial_{2,j} \psi)^q]\le K_{R,p,q}\,.
    \end{align}
    This last bound is a consequence of Assumption~\ref{assump:main}, and the fact that, on those balls, $\btheta^\top X$ is a mixture of Gaussian random vectors in $\R^{k_1}$ with mean of order $O(R)$ and variance of order $O(R^2)$.

    For the second part, consider 
    \begin{equation*}
    \mathbb E[\|\nabla L\|^8] \lesssim \sum_{a=1}^{k_1}\mathbb E[\| \nabla_{\theta^a }L\|^8] + \sum_{b=1}^{k_2} \mathbb E[|\partial_{w^b} L|^8] \,,
    \end{equation*}
    and these similarly, are bounded by 
    \begin{equation*}
        {k_1} (\max_{i\in [k_1]}\mathbb E[|\partial_{1,i} \psi|^8 \|Z_\parallel\|^8] + \mathbb E[|\partial_{1,i} \psi|^8]\mathbb E[\|Z_\perp\|^8) + {k_2} \max_{j\in [k_2]} \mathbb E[|\partial_{2,j} \psi|^8]\,.
    \end{equation*}

   For $p> 8$, letting $q$ be its Holder dual, then since $\mathbb E[\|Z_\parallel\|^{q}] \le C(p,{k_1})$ and $\mathbb E[\|Z_\perp\|^8] \le C d^4$ for a universal constant $C$, this satisfies the desired $\delta^{-4}$ bound  by applying \eqref{eq:main-assumption-on-psi}.

    \paragraph{\emph{Item 3}} For the first part of item 3, 
    we begin by bounding $\langle\nabla \Phi, \nabla u\rangle ^4\le \|\nabla \Phi\|^4\|\nabla u\|^4$ which by item~2 is at most $C \|\nabla u\|^4$, which is bounded by some $K_R$ for $(\btheta,w)\in B_R(0)$ for any of the summary statistics $u$ by item 1. Therefore, it suffices to show our claimed bound for $\mathbb E[\langle \nabla L,\nabla u\rangle ^4]$. 
    For $u=\langle \theta^a, \mu^b\rangle$,
    \begin{equation}\label{eq:I-reference-this-for-delta-loc}
    \mathbb E[ \langle \nabla L, \nabla u\rangle ^4]
     \lesssim {k_1} \big(\mathbb E[|\partial_{1,a} \psi|^4 ] + \mathbb E[|\partial_{1,a} \psi|^4 \langle Z_{\parallel}, \mu^b\rangle^4] + \mathbb E[|\partial_{1,a} \psi|^4 \langle Z_{\perp}, \mu^b \rangle ^4]\big) + \Lambda^4\langle \theta^a,\mu^b\rangle^4\,.
    \end{equation}
    Since all moments of $Z_{\parallel}$ are bounded and $\|\mu^b\| = O(1)$, the second term is controlled by a constant times $\mathbb E[|\partial_{1,a} \psi|^8]$ say. For the third term, we can use independence of $Z_\perp$ from $\partial_{1,a} \psi$ at which point we use $\mathbb E[\langle Z_\perp,\mu^b\rangle ^4] = O(1)$. Thus, under~\eqref{eq:main-assumption-on-psi}, this is bounded by some constant uniformly over $(\btheta,w)\in B_R(0)$. 
    The case of summary statistic $u=\langle \theta^a,\theta^b\rangle$ is analogous. 
       The case of summary statistic $u$ which is the second layer weights, gives 
            $\mathbb E[\langle \nabla L, \nabla u\rangle^4] \le \mathbb E[|\partial_2 \psi|^4]\,,$
       which is again bounded by an $R$-dependent constant per~\eqref{eq:main-assumption-on-psi}. 

       Lastly, for part two of item 3, the only summary statistics for which the quantity is non-zero are the non-linear ones,  namely $u= \langle \theta^a ,\theta^b\rangle$. For this, the Hessian of $u$ is an identity matrix in the $\theta^a,\theta^b$ block. Then, 
       \begin{equation*}
           \mathbb E[ \langle I_{\theta^a\theta^b}, \nabla L \otimes\nabla L - \mathbb E[\nabla L \otimes \nabla L]\rangle^2] \leq  \mathbb E[ \langle \nabla_{\theta^a} L, \nabla_{\theta^b}L \rangle^2] \lesssim \E\norm{\nabla_{\theta^a} L}^4+\E\norm{\nabla_{\theta^b} L}^4\,, 
           \end{equation*}
          which is at most  $K_R d^2$ in light of the above bound on $\E\norm{\nabla L}^8$.
\end{proof}

The following lemma shows that if a model with its summary statistics is strongly $\delta$-localizable, then the usual summary statistics for projective models are in the ballistic regime, that is, the limiting dynamics is an ODE system.

\begin{lem}\label{lem:strongly-localizable-is-ballistic}
    Suppose that a (not-necessarily Gaussian) projective model with summary statistics $\mathbf{u} = (\mathbf{G},\mathbf{w})$ is such that the triple $(\mathcal P,L, \mathbf{u})$ is strongly $\delta$-localizable. Then for each $R$, 
    \begin{equation*}
        \sup_{\Theta \in \mathbf{u}^{-1}(B_R(0))}\| \delta J V J^\top\| \to 0\,.
    \end{equation*}
\end{lem}
\begin{proof}
               We need to show that $\boldsymbol{\Sigma}$ is $0$ for this set of summary statistics. For that purpose, recall that in  item 3 of strong $\delta$-localizability for the summary statistics $(\mathbf{G},w)$, we have
           $$\sup_{(\btheta,w)\in B_R(0)} \mathbb E[\langle \nabla H, \nabla u\rangle ^4] \le K_R\,.$$
           The result then follows as $\norm{\delta J V J^\top} = O(\delta\max_u \E[\langle \nabla H, \nabla u\rangle ^2])$ and $\delta = o(1)$.
\end{proof}

\begin{proof}[\textbf{\emph{Proof of Theorem~\ref{thm:Gaussian-projective-has-effective-dynamics}}}]
        Given Lemma~\ref{lem:Gaussian-strongly-delta-localizable}, it suffices to show the asymptotic closability with the function $\mathbf{h}$ of~\eqref{eq:effective-dynamics-explicit} and with $\mathbf{\Sigma} \equiv 0$ then apply Theorem~\ref{thm:main-BGJ22}.
        If $u = \langle \theta^a,\mu^b\rangle$
          \begin{equation*}
              \mathcal A u 
              = \mathbb E[(\partial_{1,a} \psi) \langle  \mu^J,\mu^b\rangle ]+ \mathbb E[(\partial_{1,a} \psi)\langle Z,\mu^b\rangle]  + 2 \Lambda \langle \theta^a, \mu^b\rangle\,,
          \end{equation*}
          where 
           $\partial_{1,a} \psi$ is a function of $\btheta^\top X, \mathbf{G},w, y$; the law of $\langle \mu^J,\mu^b\rangle$ is a dimension independent function of $y$. Since the law of $\btheta^\top X$ is only a function of $\mathbf{G}$, so is its expectation, so the entire first expectation is only a function of $(\mathbf{G},w)$. The same is true of the second expectation, because the joint law of $\langle Z,\theta^a\rangle,\langle Z,\mu^b\rangle$ is a dimension independent function of $\mathbf{G}$, and third expectation because $\langle \theta^a,\mu^b\rangle$ is an entry of $\mathbf{G}$. The same argument applies to $u = \langle \theta^a,\theta^b\rangle$ and $u = w$. 

           Turning to the population corrector, since it takes two derivatives of the summary statistic, it is only non-zero for $u$ of the form $u(\btheta,w) = \langle \theta^a,\theta^b\rangle$. For such $u$, 
           \begin{equation*}
               \mathcal L u 
                = \mathbb E[ \partial_{1,a} \psi \partial_{1,b} \psi \langle \mu^J + Z,\mu^J + Z\rangle] - \langle \E[\partial_{1,a} \psi (\mu^J + Z)], \E[\partial_{1,b}\psi (\mu^J + Z)]\rangle \,.
           \end{equation*}
           Since the joint laws of $\langle Z,\theta^a\rangle, \langle Z,\mu^b\rangle$ for $a,b$ are a function of $\mathbf{G}$, and $\psi$ is a function of such inner products, $\mathbf{G},w$ and label $y$, this expectation is again only a function of $(\mathbf{G},w)$. In fact, the second is a dimension-independent function of $(\mathbf{G},w)$. The last term is $O(1)$ in the dimension per item 2 of $\delta$-localizability. 
            The only term that is not $O(1)$ uniformly over compacts of $(\mathbf{G},w)$ is the contribution to the first term from the inner product of $Z_{\perp}$ with itself. In particular, using independence of $Z_\perp$ of $Z^\top_\parallel \btheta$ and $Z^\top \mu^J$, we get 
           \begin{align}\label{eq:delta-Ln-for-theta^top-theta}
               \delta \mathcal L u = O(\delta) + \delta \mathbb E[ \partial_{1,a} \psi \partial_{1,b} \psi]\mathbb E[\|Z_\perp\|^2]= O(\delta) +\big(  \delta\mathbb E[\|Z_\perp\|^2]\big) \mathbb E[\partial_{1,a} \psi \partial_{1,b}\psi]\,.
           \end{align}
           Since $\delta\to 0$, the first term vanishes as $d\to\infty$ and one gets $\mathbb E[\|Z_\perp\|^2] = d(1-o(1))$ and so we conclude that for $u = \langle \theta^a,\theta^b\rangle$, one has 
           \begin{align}\label{eq:quantitative-closeness-for-delta-L_n}
               \sup_{\Theta \in \mathbf{u}_d^{-1}(B_R(0))} |\delta\mathcal L u - c_{\LR} \mathbb E[\partial_{1,a}\psi \partial_{1,b}\psi]| \le O(\delta)\,,
           \end{align}
           and it is zero for all other summary statistics $u$. Note that the limiting expectation only depends on $\theta$ through its summary statistics $\mathbf{G}$. 
           Putting these above together one has that $(-\mathcal A + \delta \mathcal L) u$ converges as $d\to\infty$ to the claimed quantity \eqref{eq:effective-dynamics-explicit} upon recalling the definition of $\cZ^{(a)}$.

           Next, Lemma~\ref{lem:strongly-localizable-is-ballistic} implies that effective volatility $\boldsymbol{\Sigma}$ is identically zero. 
           Applying Theorem~\ref{thm:main-BGJ22} yields the claimed result. 
           Finally, $\mathbf{h}$ is locally Lipschitz, and in fact differentiable, in the summary statistics $u$ because $\psi \in C^3$ with derivatives in $\cF_q$. See Lemma~\ref{lem:locally-Lipschitz} for a short justification. 
\end{proof}

We next verify that the $\delta$-localizability also allows one to probe fluctuations of the SGD trajectory 
in neighborhoods of the fixed points of~\eqref{eq:Gaussian-effective-dynamics}.

\begin{proof}[\textbf{\emph{Proof of Theorem~\ref{thm:Gaussian-projective-diffusive-limit}}}]
    We begin with verifying $\delta$-localizability still holds for the rescaled $\widetilde{\mathbf{u}}$. Note that in terms of the ballistic summary statistics $u$, one has $\nabla \widetilde u = \sqrt{d} \nabla u$. 

\paragraph{\emph{Item 1}} As in the ballistic case, the only thing to consider is the Hessian of $\tilde u$ where $\tilde u=  \sqrt{d} (\langle \theta^a,\theta^b\rangle- u_*)$, which gives us $\sqrt{d} I_d$ in a sub-block and zero elsewhere. This has operator norm $\sqrt{d}$ which is big-$O$ of $\delta^{-1/2}$ when $\delta= O(1/d)$. 

\paragraph{\emph{Item 2}} Since a ball $B_R(0)$ in the zoomed-in summary statistic space is a subset of a ball in the original scaling, this item follows as earlier. 

    \paragraph{\emph{Item 3}} By strong $\delta$-localizabilty  one has $\mathbb E[\langle \nabla H,\nabla u \rangle ^4] \le K_R$ for all $\Theta: (\mathbf{G},w)\in B_R(0)$. As a result, $\mathbb E[\langle \nabla H,\nabla \widetilde{u}\rangle ^4] \le K_R d^2$ which satisfies the desired bound.

    For part 2 of Item 3, the only summary statistics for which this is non-zero are the non-linear ones, so those $u(\Theta) = \langle \theta^a ,\theta^b\rangle$. For this, the Hessian of $\widetilde{u}$ is $\sqrt{d}$ times the identity matrix in the $\theta^a,\theta^b$ block, but just naively multiplying our bound from the strong localizability proof (see Lemma~\ref{lem:Gaussian-strongly-delta-localizable}) by $d$ for the squared Hessian does not work and we need to use a cancellation between $\nabla L$ and $\nabla \Phi$. To this end, consider
       \begin{equation*}
           d \mathbb E[ \langle I_{\theta^a\theta^b}, \nabla L \otimes\nabla L - \mathbb E[\nabla L \otimes \nabla L]\rangle^2]= d \Var(\langle\nabla_{\theta^a}L,\nabla_{\theta^b}L\rangle)\,.
           \end{equation*}
           Writing this out in terms of $\psi$, we see that, by the same arguments as above, all of the terms are $o(d^3)$ except
           \begin{equation*}
               d \Var((\partial_{1,a} \psi)(\partial_{1,b}\psi)\|Z_{\perp}\|^2)\,.
           \end{equation*}
           Observe that $(\partial_{1,a} \psi)(\partial_{1,b}\psi)$ is independent of $Z_\perp$, and the latter is standard Gaussian (in the orthogonal subspace). Since $\Var(XY)=\Var(X)\Var(Y)+\Var(X)\cdot (\E Y)^2+\Var(Y)\cdot(\E Y)^2$ for $X,Y$ independent, we see that by the same moment bounds on $\partial_{1,a}\psi$ as above  \eqref{eq:main-assumption-on-psi}, and the fact that $\Var(\norm{Z_\perp}^2)=O(d)$, we see that this is $O(d^2)=o(\delta^{-3})$.

    With the conditions for $\delta$-localizability in place, we now check the asymptotic closability and give the explicit expressions for the limiting drift and volatility functions.
    For the family of summary statistics $\mathbf{u}$, we recall the function of summary statistics from~\eqref{eq:effective-dynamics-explicit} $\mathbf{h}: \mathbb R^q \to \mathbb R^q$ given by $\mathbf{h}_u (\Theta) = \mathbf{h}_u(\mathbf{u}(\Theta))$ for different summary statistics $u$. Note that we showed there that for each $u$, 
    \begin{equation*}
       \|\mathbf{h}_u(\bu(\Theta)) - (-\cA + \delta \cL)u (\Theta)\| = O(\delta)\,,
    \end{equation*}
    uniformly on compacts of $\Theta$. 
    Then we can consider 
    \begin{equation*}
         (- \cA + \delta \cL)\widetilde u = \sqrt{d}(-\cA + \delta \cL) (u(\Theta) - u(\Theta_*)) = \sqrt{d}(\mathbf{h}_u(\mathbf{u}(\Theta)) - \mathbf{h}_u(\mathbf{u}(\Theta_*)) + O(\delta))\,.
    \end{equation*}
     Note that $\mathbf{h}\in C^2$ as $\psi\in C^5$ by Lemma~\ref{lem:locally-Lipschitz}.
    We can now Taylor expand $\mathbf{h}$ and use the assumption that $\bu_*=  \bu(\Theta_*)$ is a fixed-point of the dynamical system $d\bu = \mathbf{h}(\bu)dt$, to see that the right hand-side is 
    \begin{equation*}
        \sqrt{d}\langle \nabla \mathbf{h}_u(\bu_*), \bu(\Theta) - \bu(\Theta_*) \rangle + O(\sqrt{d}\| \bu - \bu_*\|^2) +o(1)= \langle \nabla \mathbf{h}_u(\bu_*), \widetilde {\mathbf{u}}\rangle + O(\sqrt{d}\|\mathbf{u}- \mathbf{u}_*\|^2)+o(1)\,. 
    \end{equation*}
    Uniformly on compacts of $\widetilde {\mathbf{u}}$, the  big-O term is evidently $o(1)$. 
    The first quantity is evidently dimension independent so the limit of the above is exactly~\eqref{eq:Gaussian-SDE-drift}. 
    That this is locally Lipschitz in $\widetilde{\mathbf{u}}$ follows from the fact that $\psi \in C^5$ satisfies the uniform polynomial growth of~\eqref{eq:psi-regularity} and  Lemma~\ref{lem:locally-Lipschitz}. 

    It remains to compute the volatility.  For  $J=J_d = \sqrt{d} (\bmu,\btheta, I_{{k_2}})$,
    \begin{equation*}
         \delta J V J^\top=
       \delta d   (\bmu, \btheta, I_w) \mathbb E[\nabla H\otimes \nabla H](\bmu, \btheta, I_w)^\top\,.
    \end{equation*}
    The entries of this are indexed by pairs of summary statistics. For the pair of summary statistic functions $\tilde u,\tilde u'$ for   $u=\langle \theta^a,\mu^c\rangle$ , $u'=\langle \theta^b,\mu^d\rangle$ this gives the following function of $\Theta$: 
        \begin{equation*}
         \delta d \mathbb E[ \langle \nabla_{\theta^a} H, \mu^c\rangle \langle \nabla _{\theta^b} H,\mu^d\rangle ] = \delta d \Cov(\partial_{1a}\psi \cZ^{(A)}_c,\partial_{1b}\psi \cZ^{(A)}_d)\,,
    \end{equation*}
    evaluated at the summary statistic value $\mathbf{u}(\Theta)$. This is locally Lipschitz uniformly over compacts of $\widetilde{\mathbf{u}}$ by Lemma~\ref{lem:locally-Lipschitz}, so when evaluating on a sequence of summary statistics converging to $\mathbf{u}(\Theta_*)$, we can pass to the limit on both the function, and on $d\delta$ to get $\boldsymbol{\Sigma}_{u,u'} (\mathbf{u}_*)$ from~\eqref{eq:volatility-matrix}. The same calculations for  $u = \langle \theta^a,\theta^b\rangle$ and $u=w$, yield the constant volatility matrix  \eqref{eq:volatility-matrix} as desired. 
\end{proof}

\section{Quantitative universality of functions of low-dimensional projections}\label{sec:expectations-close}
In this section, we show quantitative bounds on the difference in expectations under Gaussian noise distribution and $\nu$ noise distribution for nice functions $f$ of the vector of projections $X^\top (\btheta,\bmu)$. These bounds will be in terms of $\|\theta\|_3^3$ and will serve two key purposes: (a) to control the drift of $\|\theta\|_3^3$ under SGD with $\nu$-MM data by itself, and therefore deduce that if the parameter starts coordinate-delocalized, it stays coordinate-delocalized; and (b) to then show that if the parameter stays coordinate-delocalized on linear timescales, the limit of the summary statistic trajectories are the same as under the Gaussian data. Note that if $\theta\in {\deloc}_{\zeta}$ is coordinate-delocalized, then $\|\theta\|_3^3 \le d^{-1/2 + 3\zeta}$. 
We start 
first with the following general bound.

\begin{lem}\label{lem:3-times-differentiable}
Let $U:\R^m\to\R^r$ be matrix $[u^{(1)} \cdots u^{(r)}]$. Let $f:\R^r\to\R$ have $f\in\cF_q$ for some $q\geq 1$. 
If $Y,Z$ are i.i.d.\ vectors, mean-zero and matching second moment, and finite $(q+3)$rd moment, 
then there is constant $C>0$ depending on $\|f\|_{\mathcal F_q},r$, $\max_a \|u^{(a)}\|_2^2$ and $\E |Y_1|^{q+3}\vee \E |Z_1|^{(q+3)}$, 
such that,
    \[
    \abs{\E [f(UY)]-\E[ f(UZ)]}\leq  C\cdot\max_{a\le r} \|u^{(a)}\|_3^3 \,.
    \]
\end{lem}

\begin{proof}
The proof uses a Lindeberg replacement argument. For $j\in [m]$, define
\[
S_{-j} =\sum_{k < j} Z_ke_k + \sum_{k>j} Y_{k}e_k\,.
\]
Write the telescoping sum
\[
f(UY)-f(UZ)=\sum_j[f(U(S_{-j}+Y_je_j))-f(U(S_{-j}+Z_je_j))]\,.
\]
Let us now Taylor expand each of the two terms in the summands about $S_{-j}$ to get
\begin{multline}\label{eq:lindeberg-taylor}
f(U(S_{-j}+Y_je_j))=f(US_{-j})+\sum_a\partial_a f(US_{-j})u_j^{(a)}Y_j\\+\sum_{ab}\partial_{a,b}f(U S_{-j})u_j^{(a)}u_j^{(b)}Y_j^2+\sum_{a,b,c}\partial_{abc}f(W^Y_{j,abc})u_j^{(a)}u_j^{(b)}u_j^{(c)}Y_j^3        \,,
\end{multline}
where $W_{j,abc}^Y$ is a random variable on the line segment $(S_{-j}+tY_je_j)$ for $t\in [0,1]$ coming from the Taylor expansion remainder. A similar expression holds with $Z_j$ in place of $Y_j$.

Since $Y_j,Z_j$ are independent of $S_{-j}$, under our assumption that $Y$ and $Z$ are i.i.d vectors with first two moments matching, we see that this difference, after taking expectations and summing, is upper bounded by
\begin{equation}\label{eq:third-deriv}
\sum_{j}\sum_{a,b,c}(\E[|\partial_{abc}f(W^Y_{j,abc})||Y_j|^3]+\E[|\partial_{abc}f(W^Z_{j,abc})||Z_j|^3])|u_j^{(a)}u_j^{(b)}u_j^{(c)}|\,.
\end{equation}
Let us consider one of the summands. Observe that by $f\in\cF_q$,
\begin{align}\label{eq:partial-W-expansion-term}
\E[|\partial_{abc}f(W_{j,abc}^Y)||Y_j|^3]
&\leq \E[\sup_{t\in[0,1]|}|\partial_{abc}f(U(S_{-j}+tY_je_j))]\cdot |Y_j|^3] \nonumber \\
&\le\|f\|_{\mathcal F_q} \E[(C_q+ C_q (\norm{US_{-j}}^q+|\sum_a u_j^{(a)}|^q|Y_j|^q)|Y_j|^3] \,.
\end{align}
We will establish that the right-hand side is order one for $\max_{a}\|u^{(a)}\|_2^2 \le R$. To see this, will establish that 
\begin{equation}\label{eq:U-inner-S-bound}
\sup_{a,j}\E [|\langle u^{(a)} ,S_{-j}\rangle|^q]\leq K  \,,
\end{equation}
for some constant $K=K(q,R,\E |Y_1|^{q+1},\E |Z_1|^{q+1})$.
Before proving~\eqref{eq:U-inner-S-bound}, let us finish the proof. By~\eqref{eq:U-inner-S-bound}, we may bound~\eqref{eq:partial-W-expansion-term} by some constant depending on $r,R,\|f\|_{\mathcal F_q}$ times
\[
(1+K)\E|Y_j|^3 +\E|Y_j|^{q+3}.
\]
which is order 1 as desired. Clearly, we may obtain the equivalent bound with $Z$ in place of $Y$. 

Combining these, we find that \eqref{eq:third-deriv} is bounded by a constant depending on the $(q+3)$-rd moments of $Y,Z$ and $r,R,\|f\|_{\mathcal F_q}, q$, times
\[
\sum_j\sum_{a,b,c}|u_j^{(a)}u_j^{(b)}u_j^{(c)}|
\lesssim_r \max_{a\leq r}\norm{u^{(a)}}^3_3\,.
\]
by Young's inequality for products, as desired.

It remains to prove the claimed \eqref{eq:U-inner-S-bound}.
To that end, for $u = u^{(a)}$, write
\begin{equation*}
    \langle u , S_{-j}\rangle  = \sum_{i<j} u_i Z_i  + \sum_{i>j} u_i Y_i .
\end{equation*}
Using $(a+b)^q \lesssim_q a^q + b^q$, we will show these are $O(1)$ for each $a\in[r]$. Suppose first that $q$ is even. Then we may write

\begin{equation*}
    \E[ \langle u, Y\rangle^q] = \E \Big[\sum u_{i_1} \cdots u_{i_q}  Y_{i_1} \cdots Y_{i_q}\Big]\,.
\end{equation*}
The expectation $\E[ Y_{i_1}\cdots Y_{i_q}]$ is only non-zero if the multiset $I = \{i_1,\ldots,i_q\}$ has no singletons. Then, by Holder's inequality 
\begin{equation*}
    \E[ \langle u, Y\rangle^q] \le \sum_{I} |u_{i_1}|\cdots |u_{i_q}| \E[Y_1^q] \mathbf1\{ I \text{ has no singletons}\}\,,
\end{equation*}
which in turn is bounded by 
\begin{align}\label{eq:m-sum-IP}
    \E[|Y_1|^q] \sum_{\ell\le q/2} \sum_{\vec{d} = (d_1,\ldots,d_\ell)}  \Big( \prod_{l=1}^\ell  \sum_i |u_i|^{d_l}\Big) = \E[|Y_1|^q] \sum_{\ell \le q/2} \sum_{\vec{d}= (d_1,\ldots,d_\ell)} \prod_{l=1}^\ell \|u\|^{d_l}_{d_l}\,.
\end{align}
By inclusion of $\ell^p$ spaces and the fact that all $d_l\ge 2$, this is at most 
$q \E[|Y_1|^q]  \phi(q) \|u\|^q_2$
where $\phi(q)$ is the number of integer partitions of $q$, from which the bound follows.
If instead $q$ is odd, note that as we have $(q+3)$ moments, we may first obtain the equivalent bound for the $(q+1)$st moment and then obtain the desired bound for the $q$th moment by Jensen's inequality.
\end{proof}

We now turn to the following more specialized bound which gives an even better bound when  incorporating a single coordinate-aligned projection linearly.

\begin{lem}\label{lem:3-times-differentiable-and-one}
   Let $U$ as above and $f: \mathbb R^r \to \mathbb R$ with $f\in\cF_q$ for some $q\geq 1$ even.
    If $Y,Z$ are i.i.d.\ vectors with mean-zero, matching second moment, and finite $(q+4)$th moment, then there exists $C>0$ depending on $\|f\|_{\mathcal F_q},r, \max_{a} \|u^{(a)}\|_2^2$ and $\E[|Y_1|^{q+4}] \vee \E[|Z_1|^{q+4}]$, such that  for all $i \in [d]$,
    \[
    \abs{\E [Y_i f(UY)]-\E[ Z_i f( UZ)]}\leq  C [\max_{a\le r}|u_i^{(a)}|^2+\max_{a\le r} |u_i^{(a)}| \cdot \|u^{(a)}\|_3^3 ]\,.
    \]
\end{lem}
\begin{proof}
The proof follows a Lindeberg replacement strategy again. Define $S_{-j}$ and $R=\max_{a} \|u^{(a)}\|_2^2$ as before. Without loss of generality, take $i=1$. Then we write
\[
\E [Y_1f(U Y)]-\E[ Z_1f(UZ)]= I +II
\]
where
\begin{align*}
I&= \E Y_1f(UY)-Z_1f(U(S_{-1}+Z_1e_1))\\
II &= \E\Big[Z_1\Big(\sum_{k\geq2}f(U(S_{-k}+Y_ke_k))-f(U(S_{-k}+Z_ke_k))\Big)\Big]\,.
\end{align*}
We bound these in turn. 
We begin with $|I|$. To this end, let us expand $f$ to third order as in \eqref{eq:lindeberg-taylor} for $j=1$. 
Taking expectations and looking at the difference we obtain
\begin{align}
|\E Y_1 f(UY)-\E Z_1f(U(S_{-1}+Z_1e_1)| &\lesssim \sum_{a,b}\E|\partial_{ab}f(U S_{-1})||u_1^{(a)}u_1^{(b)}| |Y_1^3| \nonumber \\
&\quad\qquad +\sum_{2\leq abc\leq r+1}\E|\partial_{a,b,c}f(W_{abc}^Y)||u_1^{(a)}u_1^{(b)}u_1^{(c)}||Y_1^4|+(Y\mapsto Z)
\nonumber \\
&=:(i)_Y+(ii)_Y + (i)_Z + (ii)_Z\,. \label{eq:and-one-decomposition}
\end{align}
Here the $W^Y_{abc}$ again indicates the random point on the line segment $S_{-1} + tY_1 e_1$ for $t\in [0,1]$, and  $(Y \mapsto Z)$ indicates the same two sums with swapping $Y_1$'s to $Z_1$'s and $W^Y_{abc}$ to $W^Z_{abc}$.  

We now bound $|(i)_Y|$. Observe that any one of the $Y$ summands is bounded as
\[
\E (K+ K\|U S_{-1}\|^{q}) |\cdot |u_1^{(a)}u_1^{(b)}||Y_1|^3 \le r |u_1^{(a)} u_1^{(b)}| \max_{a}\E[ (K+ K|\langle u^{(a)}, S_{-1}\rangle|^q) ]\E[|Y_1|^3]  \,.
\]
By~\eqref{eq:U-inner-S-bound}, this is at most a constant (depending on $\|f\|_{\mathcal F_q}, r, R$ and the $(q+4)$th moment of $Y_1$), times $|u_1^{(a)}| |u_1^{(b)}|\le \max_{a} |u_1^{(a)}|^2$. The third derivative terms is handled analogously--in fact it is more directly the same as the bound of~\eqref{eq:partial-W-expansion-term}---whence it is at most $\max_{a,b,c} |u_1^{(a)}||u_1^{(b)}||u_1^{(c)}| \le \sqrt{R} \cdot \max_{a} |u_1^{a}|^2$. The terms $(i)_Z$ and $(ii)_Z$ are bounded symmetrically. 

Let us now turn to $II$. We begin by integrating each term by parts in $Z_1$ to get 
\[
II = \sum_a u^{(a)}_1 \sum_{k\geq 2}\E[\partial_a f(U(S_{-k}+Y_ke_k))-\partial_af(U(S_{-k}+Z_ke_k))]\,.
\]
From here we may apply Lemma~\ref{lem:3-times-differentiable} by first conditioning on $Z_1$, to obtain
\[
\abs{II}\leq K \sum_a |u^{(a)}_1|\cdot\norm{u^{(a)}}^3_3
\]
for some $K$ depending on the $(q+3)$-rd moment of $Y_1,Z_1$,  $\|f\|_{\mathcal F_q},r,$ and $R$ as desired.
\end{proof}

\subsection{The family of functions to which the bounds are applied}

As a consequence of the above, we are able to conclude that expectations of a family of relevant functions of the data's projections into the latent low-dimensional space, together with at most one extra coordinate, are small. 

Let $\mathcal G_{0}$ be the family of the following functions for $j\in [k],a,a'\in [k_1],b\in [k_2], p\in[20]$: 
\begin{align*}
     \langle \nabla_{\theta^a} L, \mu^b\rangle  &= \partial_{1,a} \psi \langle X,\mu^b\rangle\,,  
    &\langle \nabla_{\theta^a} L,\theta^{a'}\rangle  &= \partial_{1,a} \psi \langle X,\theta^{a'} \rangle\,, & 
    \nabla_{w^b} L  &= \partial_{2,b} \psi\,,\\ 
    \frac{\langle \nabla_{\theta^a }L,\nabla_{\theta^{a'}} L\rangle}{\|Y\|^2}  &= (\partial_{1,a} \psi)(\partial_{1,a'} \psi) \,,
    &(\partial_{1,a} \psi)^{p}\,,&
     &(\partial_{2,b} \psi)^p\,,
\end{align*}
and for $i \in [d]$, let $\mathcal G_{1,i}$ be the family of the following functions: 
    \[\langle \nabla_{\theta^a} L,e_i\rangle  = \partial_{1,a} \psi \langle X,e_i \rangle \qquad \text{for $a\in [k_1]$}  \,. \]
With this we have the following two corollaries. In what follows, when the noise distribution is explicitly taken to be Gaussian, we will use $\bar{\E}$ to distinguish it from $\E$ (which is typically understood to be for the $\nu$ noise distribution).

\begin{cor}\label{cor:universality-for-exact-function-families}
    Suppose Assumption~\ref{assump:main} holds, $X\sim \nu$-MM, and $\bar X\sim$GMM. For any $F\in \mathcal G_0$, there exists $C>0$ depending on $F,R,$ and the moments of $\nu$ such that for all $\btheta \in B_R(0)$, 
    \begin{equation*}
        |\E[F(X)] - \bar \E[F(\bar X)]| \le C  \big(\max_a\|\theta^{a}\|_3^3 \vee \|\mu^a\|_3^3\big)\,.
    \end{equation*}
 Similarly, for any $i \in [d]$ and any $F \in \mathcal G_{1,i}$,
    \begin{equation*}
        |\E[F(X)] - \bar \E[F(\bar X)]| \le  C  \max\{
    (\max_a|\mu_i^{a}|\vee|\theta^{a}_i|) \cdot (\max_{a}\|\theta^{a}\|_3^3\vee \|\mu^a\|_3^3), \max_{a}|\theta_i^{a} |^2\vee |\mu_i^a|^2\}\,.
    \end{equation*}
\end{cor}

\begin{proof}
    The first part is an immediate consequence of  Lemma~\ref{lem:3-times-differentiable} applied with $U = (\btheta,\bmu)$.  Assumption~\ref{assump:main} guarantees all the functions in $\mathcal G_0$ (after conditioning on $J\sim (p_a)_a$) are in $\mathcal F_{\max\{20q,q^2\}}$, which in turn requires $\max \{20q +1 +3, q^2 + 1 +3\}$ moments of $\nu$ to be finite, with the $+1$ being to make it even if it is odd.   
    
    We turn to the second part. Without loss of generality take $i =1$. Start by writing 
    \[
    \E[\partial_{1,a}\psi\langle X,e_1\rangle]=\E [\partial_{1,a}\psi\langle\mu^J,e_1\rangle]+\E[\partial_{1,a}\psi\langle Y,e_1\rangle]\,.
    \]
    and similarly for $\bar \E$. 
    For the first term, condition on the class $J$ and apply Lemma~\ref{lem:3-times-differentiable} to get
    \[
    |\E[ \partial_{1,a}\psi\langle\mu^J,e_1\rangle]-\bar\E [\partial_{1,a}\psi\langle\mu^J,e_1\rangle]|\leq\max_a|\mu^a_1|\big(\max_a\norm{\theta^{a}}_3^3 \vee\|\mu^a\|_3^3\big)\,.
    \]
    For the second difference, we may apply Lemma~\ref{lem:3-times-differentiable-and-one}, which requires $q+1+4 \le 20q+4$  moments of $\nu$ to be finite. Combining these two bounds we obtain the claimed result. 
\end{proof}

    \begin{cor}\label{cor:specific-LR-function}
        For the specific function $F = \langle \nabla _{\theta^a} L , \nabla_{\theta^b} L\rangle = (\partial_{1,a} \psi )(\partial_{1,b}\psi) \|X\|^2$, which does not fit into the above form, if $\delta = O(1/d)$, one has uniformly over $\btheta \in B_R(0)$, 
        \begin{equation*}
            \delta |\E[ F(X)] - \bar \E[F(\bar X)]| \le C (\max_a \|\theta^{a}\|_3^3\vee \|\mu^a\|_3^3 + d^{-1/2})\,.
        \end{equation*}
    \end{cor}

\begin{proof}
We begin by conditioning on class $J$ and writing  
\begin{equation*}
    F(X) = (\partial_{1,i}\psi)(\partial_{1,j}\psi )( \|\mu^J\|^2 + 2 \langle Y,\mu^J\rangle + \|Y\|^2)\,,
\end{equation*}
where $Y \sim \nu^{\otimes d}$. 
By Corollary~\ref{cor:universality-for-exact-function-families} and $\|\mu^J\|^2 = O(1)$, the first two terms are $O(1)$ uniformly over $\Theta \in B_R(0)$, which since $\delta= O(1/d)$ gives an $O(d^{-1})$ error term. Turning to the third term, we write it as 
\begin{equation*}
    \E[ (\partial_{1,i}\psi) (\partial_{1,j}\psi) (\|Y\|^2 -d)]  + d \E[ (\partial_{1,i}\psi) (\partial_{1,j}\psi)]\,.
\end{equation*}
By Corollary~\ref{cor:universality-for-exact-function-families}, the second expectation is within $O(\max_a\|\theta^{a}\|^3_3\vee \|\mu^a\|_3^3)$ of the Gaussian expectation $d \bar \E [(\partial_{1i}\psi ) (\partial_{1j}\psi)]$ and the factor of $d$ gets canceled out with the multiplication by $\delta = O(1/d)$. For the first term, by Cauchy--Schwarz, it is at most 
\begin{equation*}
    \E[(\partial_{1,i}\psi)^4]^{1/4}\E[(\partial_{1,j}\psi)^4]^{1/4} \E[ (\|Y\|^2 - d)^2]^{1/2}\,.
\end{equation*}
By Corollary~\ref{cor:universality-for-exact-function-families}, the first two terms are $O(1)$ uniformly over $\Theta \in B_R(0)$. For the third term, since $\nu$ has finite fourth moments, we can compute explicitly, 
$$\E[(\|Y\|^2 -d)^2]  = d \E[Y_1^4] -d =O(d)\,.$$ 
Taking square root and multiplying by $\delta$, this gives an $O(d^{-1/2})$ contribution. Altogether, this shows that $\delta \E[F(X)]$ is within $O (\max_{a} \|\theta^a\|_3^3 \vee \|\mu^a\|_3^3 + d^{-1/2})$ of $c_{\LR} \bar{\E}[(\partial_{1,i}\psi) (\partial_{1,j}\psi)]$. The same argument shows that $\delta \E[F(\bar X)]$ is also within that distance of $O (\max_{a} \|\theta^a\|_3^3 \vee \|\mu^a\|_3^3 + d^{-1/2})$ of $c_{\LR} \bar{\E}[(\partial_{1,i}\psi) (\partial_{1,j}\psi)]$. Together with the triangle inequality we get the claim. 
\end{proof}

\section{The parameter stays coordinate-delocalized under SGD}\label{sec:parameter-stays-coordinate-delocalized}

The main aim of this section is to establish that even under $\nu$-MM data model,  if the initialization for SGD is coordinate-delocalized, then the parameter trained under SGD remains coordinate-delocalized for all linear timescales. This will only be proved within compact sets of the parameter space so we introduce the exit time of SGD from the $R$-ball of $\mathbb R^p$:
\begin{align}\label{eq:tau-R}
    \tau_R := \min\{\ell: \|\Theta_{\ell}\|_2^2 \ge R\}\,.
\end{align}

\begin{thm}\label{thm:stays-coordinate-delocalized}
    Consider SGD for loss satisfying Assumption~\ref{assump:main} with respect to $\nu$ noise distribution.  For $\zeta < 1/8$, there exists $\varepsilon>0$ such that for any $R>0$ the following holds: If $\Theta_0 \in {\deloc}_{\zeta} \cap B_R(0)$ and $\bmu \in {\deloc}_{\zeta} $, then  
    \begin{equation*}
        \mathbb P(\exists \ell\in [0, (\varepsilon \delta^{-1} \log d) \wedge \tau_{R}]:  \Theta_{\ell} \notin {\deloc}_{\sfrac{1}{8}}) =  o(1)\,.
    \end{equation*}
\end{thm}

Also in this section, we deduce Theorem~\ref{thm:dgec} as a consequence of Theorem~\ref{thm:stays-coordinate-delocalized}.

\subsection{Ballistic $\delta$-localizability on the coordinate-delocalized set}

We start by using the results of Section~\ref{sec:expectations-close} to show that  strong $\delta$-localizability
also holds for the $\nu$-MM data when $\Theta$ is coordinate-delocalized. 

\begin{lem}\label{lem:delta-localizability-on-good}
    Suppose we have a projective model $(L,\mathcal P_X)$ where $\mathcal P_X$ is a $\nu$-MM satisfying Assumption~\ref{assump:main}. Assume that the ground truth vectors $\mu^a \in {\deloc}_{\zeta}$ with $\|\mu^a\|=O(1)$. For every $\zeta\le 1/8$, the triple $(\mathbf{u}, L, P)$ with summary statistics $\mathbf{u} = (\mathbf{G},w)$ is strongly $\delta$-localizable on ${\deloc}_{\zeta}$. 
\end{lem}

\begin{proof}
    Item 1 was already shown in Lemma~\ref{lem:Gaussian-strongly-delta-localizable} as the data distribution plays no role. For the remaining items, the big distinction from the Gaussian case is that we cannot split $Z$ into $Z_{\parallel}$ and $Z_\perp$ which are independent. This is where we will use the fact that we are on ${\deloc}_{\zeta}$ and the closeness of expectations of functions of $\btheta^\top X$.

\medskip
    \noindent \emph{Item 2}.
     The delicate term to handle is the first part of item 2. For instance, for $\Theta \in {\deloc}_{\zeta}$, consider 
    \begin{equation*}
        \|\nabla_{\theta^a} \Phi\|^2 = \|\mathbb E[(\partial_{1,a} \psi) X]\|^2 + 4 \Lambda^2 \|\theta^a\|_2^2\,.
    \end{equation*}
    The second term is clearly $O(1)$ uniformly over $\Theta \in B_R(0)$. 
    If we let $\bar\E$ be expectation with respect to GMM data, by Corollary~\ref{cor:universality-for-exact-function-families} applied to $F \in \mathcal G_{1,\iota}$, one has for all $\Theta \in {\deloc}_{\zeta} \cap B_R(0)$ and $\bmu\in {\deloc}_{\zeta}$, 
    \begin{align*}
        \|\mathbb E[(\partial_{1,a} \psi) X]\|^2 = \sum_{i} \E [(\partial_{1,a} \psi) \langle X,e_i\rangle  ]^2 & = \sum_{i} \big(\bar \E[(\partial_{1,a} \psi) \langle \bar X,e_i\rangle ] +  O( d^{-1 + 4\zeta}) \big)^2   \\ 
        & \le 2 \|\bar \E [(\partial_{1,a}\psi) \bar X]\|^2 +  O( d^{-1 +  8 \zeta})\,,
    \end{align*}
    which for $\zeta\le 1/8$ is bounded on compacts of $\Theta \in {\deloc}_{\zeta}$, as the Gaussian term was already bounded in Lemma~\ref{lem:Gaussian-strongly-delta-localizable}. 
    The other term to handle for the first part of item 2 is 
    \begin{equation*}
        \|\nabla_{w} \Phi\|^2 = \sum_{a} \E[\partial_{2,a} \psi]^2 \le 2 \sum_a \bar \E[ \partial_{2,a} \psi]^2 + O(d^{-1 + 3\zeta})\,,
    \end{equation*}
    where we used Corollary~\ref{cor:universality-for-exact-function-families} for $F = \partial_{2,a} \psi \in \mathcal G_0$. This is $O(1)$ as the Gaussian term was already handled in Lemma~\ref{lem:Gaussian-strongly-delta-localizable}.

    For the second half of item 2, using the previous bound on $\|\nabla \Phi\|^8$, it suffices to bound 
    \begin{equation*}
        \mathbb E[\|\nabla L\|^8] \lesssim  \sum_{a} \E[(\partial_{1,a} \psi)^8 \langle X,X\rangle ^4] + \sum_{a} \E[(\partial_{w_a} \psi)^8]\,.
    \end{equation*}
    For the first of these terms, by Cauchy--Schwarz, 
    \begin{equation*}
         \E[(\partial_{1,a}\psi)^8 \langle X,X\rangle^4] \le \E[(\partial_{1,a} \psi)^{16}]^{1/2} \E[ \langle X,X\rangle^8]^{1/2}\,.
    \end{equation*}
    If $\Theta \in {\deloc}_{\zeta}\cap B_R(0)$,  by Corollary~\ref{cor:universality-for-exact-function-families}, the first term is within $o(1)$ of its Gaussian expectation, which is $O(1)$ uniformly over compacts of $\Theta$ by polynomial growth of derivatives of $\psi$ at infinity per Assumption~\ref{assump:main}. The quantity $\E[\langle X,X\rangle^8]^{1/2}$ is bounded for $\nu$ by $O(d^{4}) = O(\delta^{-4})$ by expanding the inner product and using that $\nu$ has at least $16$ moments per Assumption~\ref{assump:main}. 

\medskip
\noindent \emph{Item 3}. 
    For the first part of item 3, since $\langle \nabla \Phi, \nabla u\rangle^4 \le \|\nabla \Phi\|^4 \|\nabla u\|^4$, by item 2 of $\delta$-localizability and the fact that in compacts, $\|\nabla u\|^4 \le C_R$,  it is equivalent to show that 
    \begin{equation*}
        \sup_{\Theta \in B_R(0) \cap {\deloc}_{\zeta}} \mathbb E[ \langle \nabla L ,\nabla u\rangle^4 ] =O(1)\,. 
    \end{equation*}
Expanding this out similarly to the proof of Lemma~\ref{lem:Gaussian-strongly-delta-localizable}, if $u = \langle \theta^a, \mu^b\rangle$, by Cauchy--Schwarz,  
\begin{align*}
	\mathbb E[\langle \nabla L,\nabla u\rangle^4] \lesssim \mathbb E[ |\partial_{1,a} \psi|^4 \langle X ,\mu^b\rangle^4] +  \Lambda^4 \langle \theta^a, \mu^b\rangle^4  
	 \le \mathbb E[|\partial_{1,a} \psi|^8]^{1/2} \mathbb E[ \langle X, \mu^b\rangle^8]^{1/2} +  \Lambda^4 \|\theta^a\|^4 \|\mu^b\|^4\,.
\end{align*}
The expectation $\E[|\partial_{1,a} \psi|^8]$ is $O(1)$ on $\Theta \in {\deloc}_{\zeta} \cap B_R(0)$ as it was for item 2; the expectation $\E[\langle X,\mu^b\rangle^8]$ is $O(1)$ by Cauchy--Schwarz and the fact that $\nu$ has at least $16$ moments, using $\langle u+v ,w\rangle^8 \lesssim \langle u,w\rangle^8 + \langle v,w\rangle^8$. The last term is $O(1)$ uniformly on compacts $\Theta \in B_R(0)$. 

If the summary statistic $u$ is $\langle \theta^a,\theta^b\rangle$, the argument is identical, with the bound of $\|\mu^b\|=O(1)$ replaced by the fact that we work on compacts of $\Theta$. If the summary statistic $u$ is $w^a$, then the expectation is $\E[(\partial_{2,a} \psi)^4]$ which is within $O(1)$ of its Gaussian expectation per Corollary~\ref{cor:universality-for-exact-function-families}, and that Gaussian expectation is $O(1)$ by Assumption~\ref{assump:main}. 

Finally, for the second part of item 3, we only need to check it for $u=\langle \theta^a,\theta^b\rangle$ as it is the only non-linear summary statistic. For this one,  the first step of the Gaussian bound on it from proof of Lemma~\ref{lem:Gaussian-strongly-delta-localizable} did not use Gaussianity, so we arrive identically at 
       \begin{equation*}
           \mathbb E[ \langle I_{\theta^a\theta^b}, \nabla L \otimes\nabla L - \mathbb E[\nabla L \otimes \nabla L]\rangle^2] 
           \lesssim k_1\max_{a\in [k_1]} \mathbb E[ (\sum_\ell |\partial_{1,a} \psi|^2  X_\ell^2 )^2] \,.
           \end{equation*}
           This term is at most $\E[ (\partial_{1,a} \psi)^8]^{1/2} \E[\|X\|^8]^{1/2}$ which is bounded as above by $O(d^{2}) = o(\delta_d^{-3})$.
\end{proof}

\subsection{The parameter stays coordinate-delocalized for linear timescales}
The drift for the inner product $m_{ai} = \langle \theta^a,e_i\rangle$ is given by $\langle \nabla L(\Theta,X), \nabla m_{ai}\rangle$. The following lemma shows that this drift is quantitatively small if $\Theta$ is coordinate-delocalized.

\begin{lem}\label{lem:drift-of-ith-coefficient}
    Under Assumption~\ref{assump:main},  for any $\zeta\le 1/8$, $i\in [d],$ and $R>0$, 
    there exists $C$ such that for all $\bmu \in {\deloc}_{\zeta}$, 
    \begin{equation*}
        \sup_{\Theta \in B_R(0) \cap {\deloc}_{\zeta}}|\mathbb E[\langle \nabla L(\Theta, X), \nabla m_{ai}\rangle] | \le C \max_b |m_{bi}| + O(d^{-1+4\zeta})\,.
    \end{equation*}
\end{lem}

\begin{proof}
    In what follows, let $\bar{\E}$ denote the Gaussian mixture expectation to distinguish from the $\nu$-MM expectation. Since $$\langle \nabla L, \nabla m_{ai}\rangle = (\partial_{1a}\psi ) \langle X,e_i\rangle  + 2\Lambda \langle \theta^a,e_i\rangle  = (\partial_{1,a}\psi) \langle X,e_i\rangle + 2\Lambda m_{ai}\,,$$
    we consider the first term, 
    which falls in the family of functions $\mathcal G_{1,i}$. By Corollary~\ref{cor:universality-for-exact-function-families}, 
    \begin{equation*}
        |\E[(\partial_{1,a}\psi)\langle X,e_i\rangle ] - \bar{\E}[(\partial_{1,a}\psi) \langle X,e_i\rangle ]| \le C_R \Big(|\langle\theta^a, e_i\rangle |(\max_b \|\theta^b\|_3^3\vee \|\mu^b\|_3^3)\vee |\langle\theta^a, e_i\rangle |^2\Big)\,.
    \end{equation*}
    On $\Theta, \bmu \in {\deloc}_{\zeta}$, the right-hand side here is $O(d^{-1 + 4\zeta})$. Now consider the Gaussian expectation, 
    \begin{equation*}
        \overline{\mathbb E} [\langle\nabla _{\theta^a}  L, e_i\rangle]  = \overline{\mathbb E}[ \partial_{1,a}\psi \langle \mu^J + Z, e_i\rangle ]  
         = \overline{\bE}[ \partial_{1,a}\psi \langle \mu^J,e_i\rangle] + \overline{\mathbb E} [ \partial_{1,a}\psi \langle Z_{\parallel},e_i\rangle ] + \overline{\mathbb E} [ \partial_{1,a}\psi \langle Z_{\perp},e_i\rangle ] \,.
    \end{equation*}
    Here, as before, $Z_{\parallel}$ is the projection of $Z \sim \mathcal N(0,I_d)$ into  $\Span\{\btheta,\bmu\}$.  The third term is zero by independence of $Z_\perp$ and $\partial_{1,a} \psi$. The second term is Taylor expanded into 
    \begin{equation*}
        \overline{\E} [\partial_{1,a} \psi \langle Z_{\parallel},e_i\rangle ] = \overline{\E} [ \partial_{1,a} \psi^{\perp,i} \langle Z_{\parallel},e_i\rangle] + \sum_{b} m_{bi} \overline{\E}[ \langle \nabla_{1,b} [\partial_{1,a} \psi^{\perp,i} \langle Z_{\parallel},e_i\rangle^2)],e_i\rangle ] + O(\max_{b} m_{bi}^2)
    \end{equation*}
    where $\partial_{1,a}\psi^{\perp,i}$ means evaluating on the point $\btheta$ without the part of $\btheta$ in the $e_i$ direction: $\partial_{1,a}\psi((\btheta - \btheta_i e_i)^\top X) $. 
    By Gaussianity, $\langle Z_{\parallel} , e_i\rangle$ is independent of $\langle Z_{\parallel} , (\theta^a - m_{ai} e_i)\rangle$, so the first expectation is zero. For the second expectation, by the same independence, 
    \begin{equation*}
        |\overline{\E} [ \partial_{1,b}\partial_{1,a} \psi^{\perp,i}] \overline{\E}[ \langle Z_{\parallel},e_i\rangle^2 \langle X,e_i\rangle]| \le \overline{\E} [| \partial_{1,b}\partial_{1,a} \psi^{\perp,i}|] \overline{\E} [ \langle Z_{\parallel} ,e_i\rangle^4]^{1/2} \overline{\E}[\langle X,e_i\rangle^2]^{1/2}
     \end{equation*}
    which, since $\psi\in\cF_q$ for some $q\geq 1$ and $\|\mu^J\|=O(1)$, $\|e_i\|=1$, is bounded by a constant $C(R)$.   
\end{proof}

We now deduce Theorem~\ref{thm:stays-coordinate-delocalized} that the parameter vector run under SGD with any $\nu$-MM data stays coordinate-delocalized for linear timescales if it begins coordinate-delocalized. 

\begin{proof}[\textbf{\emph{Proof of Theorem~\ref{thm:stays-coordinate-delocalized}}}]
    Let $\tau_R$ be as in~\eqref{eq:tau-R}, and for $\zeta <\zeta_0 \le  1/8$, let $\tau_{\zeta_0} = \min\{\ell: \Theta \notin {\deloc}_{\zeta_0}\}$, and $\tau = \min\{T \delta^{-1}, \tau_R, \tau_{\zeta_0}\}$.      
    Consider the evolution of the parameter's $\ell^3$ norm. For any $i\in [d]$, 
    \begin{equation*}
        m_{ai}(\ell)= m_{ai}(0) + \sum_{k\le \ell} \delta \langle \nabla L(\Theta_{k-1},X^k), \nabla m_{ai} \rangle  
         = m_{ai}(0)  + \sum_{k\le \ell} \delta \E[ \langle\nabla L ,\nabla m_{ai}\rangle] + M_\ell^a
    \end{equation*}
    where $M_\ell^a$ is a martingale. Take absolute values of both sides. Then using Lemma~\ref{lem:drift-of-ith-coefficient} to replace the expectation above up to an $O(d^{-1+4\zeta_0})$ error, for $\ell\le \tau$, 
    \begin{equation*}
        \max_{a} |m_{ai}(\ell) |\le  C \max_{a}|m_{ai}(0)| +  C\delta\sum_{k\le \ell} \max_{a}|m_{ai}(k)| + O(d^{-1+4\zeta_0}) + \max_{a}\max_{k\le \ell} |M_k^{a}|\,.
    \end{equation*}
   Bound the martingale term by applying Doob's maximal inequality, the Burkholder--Davis--Gundy  inequality, and  Minkowski's integral inequality in turn to find that for each $a\in [k_1]$ and  $p\geq1$,
    \begin{equation*}
        \mathbb P( \max_{\ell \le \tau }|M_\ell^a| >r )   \le \frac{\E[ |M_{T\delta^{-1}}^a|^p]}{r^p} \lesssim r^{-p} \E[ (\sum (M_{\ell} - M_{\ell-1})^2)^{p/2}]
        \leq r^{-p}\big(\sum_{\ell}\E[  |M_{\ell-1} -M_{\ell}|^p]^{2/p} \big)^{p/2} 
    \end{equation*}
    which is equal to
    \begin{equation*}
         r^{-p} \delta^p (\sum_{\ell} \E[|\langle \nabla L,\nabla m_{ai} \rangle - \E[\langle \nabla L,\nabla m_{ai }\rangle] |^{p}]^{2/p})^{p/2}\,.
    \end{equation*}
    We now bound the $p$'th moment of the increments of the martingale uniformly over $\Theta \in {\deloc}_{\zeta_0}\cap B_R(0)$:  
    \begin{equation*}
        \E[|\langle \nabla L,\nabla m_{ai } \rangle - \E[\langle \nabla L,\nabla m_{ai }\rangle] |^{p}] \lesssim \E[|\partial_{1,a} \psi|^p |\langle X,e_i \rangle|^p]\le \E[|\partial_{1,a} \psi|^{2p}]^{1/2} \E[ |\langle X,e_i \rangle|^{2p}]^{1/2}\,.
    \end{equation*}
    For any $p\le 10$, the first term on the right-hand-side is $O(1)$ uniformly over $\Theta \in {\deloc}_{\zeta_0}\cap B_R(0)$ by Corollary~\ref{cor:universality-for-exact-function-families}; and since we assume at least $20$ moments on $\nu$, the second term is also seen to be $O(1)$. Plugging in $r = d^{-\xi} d^{-1/2+ \zeta_0}$ for $\xi$ to be determined sufficiently small, and using that $\delta = O( 1/d)$ and $T = O(\log d)$, for $p= 10$, 
    \begin{equation}\label{eq:stays-deloc-step-1}
        \mathbb P( \max_{\ell \le \tau }|M_\ell^a| >  d^{-1/2+ \zeta_0-\xi} ) \le O(r^{-p} \delta^{p/2} (\log d)^{p/2}) = O(d^{10\xi} d^{-10 \zeta_0} (\log d)^{5})\,.
    \end{equation}

    Using a union bound over $i\in [d]$ and $a\in [k_1]$, we get that with probability $1-\tilde O(d^{10\xi} d^{-10\zeta_0 +1})$---which is $1-o(1)$ if $\zeta_0 > 1/10$ and $\xi$ is sufficiently small---the following holds 
    for all initializations in ${\deloc}_{\zeta} \cap B_R(0)$ for $d$ sufficiently large:  For all time steps $\ell \le \tau$, and all $i\in [d]$, 
    \begin{equation*}
        \max_a|m_{a i} (\ell)| \le 2 d^{-1/2 +\zeta_0-\xi} + C \delta \sum_{k\le \ell} \max_{a} |m_{a i}(k)|\,.
    \end{equation*}
    Here, the extra $d^{-1/2+\zeta_0-\xi}$ was used to absorb both $\max_{a} |m_{ai}(0)|$ and the $O(d^{-1+4\zeta_0})$ error since $\zeta<\zeta_0 \le 1/8$ and $\xi$ is sufficiently small. 
    Then by the discrete Gronwall inequality,  
    we deduce that for $\zeta<\zeta_0$ with $\zeta_0\in(1/10,1/8]$,  with probability $1-o(1)$, if the initialization and $\bmu$ are  in ${\deloc}_{\zeta}$, then  for all $i\in [d]$, and all $a\in [k_1]$
    \begin{equation}\label{eq:stays-deloc-step-3}
         \sup_{\ell\le \tau} |\theta^a_i (\ell)| \le  2 d^{ - 1/2 + \zeta_0-\xi}  e^{C' T}\,.
    \end{equation}
    Taking any $\zeta_0 \in (\zeta\vee \frac{1}{10},\frac{1}{8}]$, for $T= \varepsilon \log d$, with $\varepsilon$ sufficiently small, this is bounded by $d^{ - 1/2 + \zeta_0}$.
\end{proof}

\begin{proof}[\textbf{\emph{Proof of Theorem~\ref{thm:dgec}}}]
   By Lemma~\ref{thm:stays-coordinate-delocalized}, with probability $1-o(1)$, one has $\Theta_\ell , \bar \Theta_\ell \in \mathcal D_{1/8}$ for all $\ell \le \varepsilon d\log d \wedge \tau_R\wedge \bar \tau_R$. This uniformly bounds differences of expectations for functions of $\btheta^\top X$ with $\|f\|_{\mathcal F_0}\le 1$  by $\|\btheta\|_3^3 \le d^{-1/2+ 3\zeta} = o(1)$ by Lemma~\ref{lem:3-times-differentiable} under Assumption~\ref{assump:main}. This concludes the proof because $\|f\|_{C^3_b}= \|f\|_{\mathcal F_0}$. 
\end{proof}

\section{Universality of ballistic dynamics}\label{sec:main-ballistic-universality}

In this section, we combine the ingredients from the previous sections to establish our main Theorem~\ref{thm:ballistic-universality}. We then end the section with the simple example where initializations that are not coordinate-delocalized lead to different drifts under Gaussian vs.\ non-Gaussian distribution.  

\begin{proof}[\textbf{\emph{Proof of Theorem~\ref{thm:ballistic-universality}}}]

    Much of the beginning of this argument follows that of the proof of Theorem 2.3 from~\cite{BGJ22} so we will frequently reference that argument and explain only what changes. 
    
    Let $\Theta_\ell$ denote the evolution of SGD with data drawn from the $\nu$-mixture model and $\bar \Theta_\ell$ the SGD with Gaussian mixture data. We couple these processes so that $\bar\Theta_0=\Theta_0$. Correspondingly define $\cA,\cL$ to be the operators from~\eqref{eq:A-L-operators} defined with expectation with respect to $\nu$ and $\bar\cA,\bar\cL$ those for the Gaussian model.
    Note that $\Theta_0\in {\deloc}_{\zeta}$ for a $\zeta <1/8$. Consider $\mathbf{u}=(\bfG,w)$. 
    Let $\tau_R$ now denote the exit time for $\bu(\Theta_\ell)$ to escape $B_R(0)$ and observe that $\ell\le \tau_R$ implies $\Theta_\ell \in B_R(0)$, while $\ell> \tau_R$ implies $\Theta_\ell \notin B_{\sqrt{R}}(0)$, since  the $2$-norm of $\Theta$ is captured by the diagonal elements of $\mathbf{G}$. 
    
    Let $\tau_{\text{deloc}}$ denote the exit time of $\Theta_\ell$ from ${\deloc}_{\sfrac{1}{8}}$, the ``coordinate-delocalized set'' from Definition~\ref{def:delocalized-set}. Let $\tau=\tau_R\wedge \tau_{\text{deloc}}$. 
    Finally, fix $T$ to be a final continuous time horizon; that is, we run SGD for $T\delta^{-1} $ iterations. By Theorem~\ref{thm:stays-coordinate-delocalized}, we have for every $R,T$ that  $\tau_R\wedge T\delta^{-1}<\tau_{\text{deloc}}$ except with probability $o(1)$. We also define the same stopping times for the Gaussian model, which we denote the same but with an overbar (e.g., $\bar\tau$) and similarly have $\bar \tau_R \wedge T\delta^{-1} \le \bar \tau_{\text{deloc}}$ except with probability $o(1)$.

    Let $f$ denote one of the summary statistics in $\bu$. Note that $f$ is an at most quadratic function and thus smooth. Observe that the Doob decomposition for $f_\ell=f(\Theta_\ell)$ is of the form,
    \[
    f_\ell =f_{0}+\delta A_\ell +\delta  M_\ell
    \]
    where $A$ is adapted and $ M$ is a martingale, and their increments take the following form: 
    \begin{align*}
        A_\ell-A_{\ell-1} &= (-\cA +\delta \cL)f_{\ell-1}+\frac{\delta}{2} \langle \nabla \Phi \tensor \nabla \Phi ,\nabla^2 f\rangle_{\ell-1}\,,\\
        M_\ell-M_{\ell-1} &= \langle\nabla H^\ell,\nabla f\rangle_{\ell-1}+\delta(\cE_\ell-\cE_{\ell-1})\,,\\
        \cE_\ell-\cE_{\ell-1}&=\nabla^2 f(\nabla \Phi,\nabla H^\ell)+\frac{1}{2}\langle\nabla^2f,\nabla H^\ell\tensor\nabla H^\ell-V\rangle_{\ell-1}.
    \end{align*}
    Here, $\mathcal A, \mathcal L, H,V$ are all as in Section~\ref{sec:Gaussian-projective-models}. 
    Note in particular that $\cE_\ell$ is a martingale. 

    We begin by arguing that we may rewrite the above as 
    \begin{equation}\label{eq:standard-form}
    f_\ell =f_{\ell-1}+\delta(-\cA+\delta\cL)f_{\ell-1}+ \varepsilon_\ell\,,
    \end{equation}
    where the error term  $\varepsilon_\ell$ has
    \begin{align}\label{eq:standard-form-eps}
    \sup_{\ell\leq T\delta^{-1} \wedge \tau}|\sum_{i\le \ell}\varepsilon_i| \to 0 \quad \text{ in }L^2. 
    \end{align}
  As we only consider $\ell\leq \tau$, it suffices to prove this for the stopped versions of these processes.

    For the adapted process, $A_\ell$, the second term vanishes uniformly in $L^2$ by the same reasoning as in \cite{BGJ22} (see two displays after Eq.~6.3) and Doob's maximal inequality.

    We now show that the martingale term is negligible. 
    To this end, first note that for the first term,  the same bound in \cite[Eq.~6.4]{BGJ22} applies. In fact, due to part 1 of item 3 of strong $\delta$-localizability (which holds for all $\ell \le \tau\wedge T \delta^{-1}$ by Lemma~\ref{lem:delta-localizability-on-good}), we can improve that bound to 
    \[
    \E [(\delta^2\sum_{\ell \le \tau \wedge T\delta^{-1}} \langle\nabla H,\nabla f\rangle_{\ell-1}^2)^2]\leq(\delta\sum(\delta^2\E\langle\nabla H,\nabla f \rangle_{\ell-1}^4)^{1/2})^2\lesssim_R \delta^2T.
    \]
    Next we deal with each part of the error term, $\cE_\ell$, in turn. For the first part of the error term, we may apply \cite[Eq.~6.5]{BGJ22} directly. For the second part, by strong localizability, we can improve \cite[Eq.~6.6]{BGJ22} as well to get
    \begin{equation*}
    \E[(\delta^4\sum\langle\nabla^2f,\nabla H\tensor\nabla H-V\rangle_{\ell-1}^2)^2]
    \lesssim_R \delta^2T \,.
    \end{equation*}
    Combining these, we see that 
    $\max_{\ell\leq  T \delta^{-1} \wedge\tau} |M_\ell|\to 0$
    in $L^2$ by Doob's maximal inequality.
    Thus, combining with the above, we obtain \eqref{eq:standard-form} with $\varepsilon_\ell$ satisfying~\eqref{eq:standard-form-eps}.

    Clearly the same argument applies to the evolution of the Gaussian versions $\bar f_\ell=f(\bar\Theta_\ell)$. Hence, if we define $\Delta_\ell^f= f_\ell-\bar f_\ell$, we obtain
    \[
    \Delta_{k}^f
    =\delta\sum_{\ell\le k}[- (\cA f_\ell-\bar\cA\bar f_\ell)+\delta(\cL f_\ell-\bar\cL \bar f_\ell )]+o(1)\,,    \]
    where the $o(1)$ term tends to zero in $L^2$ uniformly for $k\leq T \delta^{-1} \wedge \tau\wedge\bar\tau$.
    We will now show that $\Delta_k^f$ is uniformly small in time 
    via a Gronwall bound.

    By Corollary~\ref{cor:universality-for-exact-function-families}, we have that uniformly over all $\Theta \in B_R(0)$,
    \[
    |\mathcal{A} f(\Theta)-\bar{\mathcal{A}} f(\Theta)| \lesssim_R \max_a \norm{\theta^a}_3^3\vee \|\mu^a\|_3^3 + O(d^{-1/2})\,.
    \]
    Therefore for any $\Theta,\bar\Theta\in {\deloc}_{\sfrac{1}{8}}\cap B_R(0)$, we have 
    \[
    |\cA  f (\Theta)-\bar{\cA}  f (\bar \Theta)| \lesssim_{R}  d^{-1/2+3/8} + \norm{\bar{\cA} f}_{
    \text{Lip}(B_R)}\norm{\mathbf{u}(\Theta)- \mathbf{u}(\bar \Theta)}\,,
    \]
    where $\|\cdot \|_{\text{Lip}(B_R)}$ denotes the Lipschitz constant of $\bar {\mathcal A} f$ viewed as a function  on $\cS=\sP_{k_1+k}\times\R^{k_2}$, the space of summary statistics of $\Theta$ (recall that the Gaussian $\mathcal A f$ is only a function of the summary statistics from~\eqref{eq:effective-dynamics-explicit}). Note that evaluated on $\Theta_\ell$, the last distance between the summary statistics is bounded by $\max_{f\in \mathbf{u}} |\Delta_\ell^f|$. 
    
    For the operator $\delta \mathcal L$, we have $\delta\cL f$ is only non-zero if $f = \langle \theta^a,\theta^b\rangle$ for some $a,b$. Let $\mathcal Bf = \delta \mathcal Lf, \bar {\mathcal B} f= \delta \bar{\mathcal L}f$ and, recalling \eqref{eq:delta-Ln-for-theta^top-theta},  
    $\bar{\mathcal B}_\infty f = c_{\LR} \bar{\E}[(\partial_{1a}\psi )(\partial_{1b}\psi)]$. By Corollary~\ref{cor:specific-LR-function} 
    \begin{align*}
        |\mathcal Bf(\Theta) - \bar{\mathcal B}f(\Theta)| \lesssim_R \max_a \|\theta^a\|_3^3 \vee \|\mu^a\|_3^3\,. 
    \end{align*}
    Since the right-hand side of this is $O(d^{-1/2+3/8})$, by the triangle inequality, 
    \begin{align*}
        |\mathcal B f(\Theta) - \bar{\mathcal{B}}f (\bar{\Theta})| \le O(d^{-\frac{1}{2} + \frac{3}{8}}) + |\bar{\mathcal B}f(\Theta)- \bar{\mathcal B}_\infty f(\Theta) | + |\bar{\mathcal B}_\infty f(\bar{\Theta})- \bar{\mathcal{B}} f(\bar\Theta)| + |\bar{\mathcal B}_\infty f(\Theta) - \bar{\mathcal B}f_\infty (\bar \Theta)|\,.
    \end{align*}
    The second and third terms are $O(\delta)=o(1)$ uniformly over $\Theta,\bar\Theta \in B_R(0)$ by~\eqref{eq:delta-Ln-for-theta^top-theta}. The fourth term is bounded by $\|\bar{\mathcal B}_\infty f\|_{\text{Lip}(B_R)} \|\mathbf{u}(\Theta) - \mathbf{u}(\bar \Theta)\|$. For all $f$, this Lipschitz constant and that of $\|\bar{\mathcal A} f\|_{\text{Lip}(B_R)}$ are $O(1)$ by recalling their forms as functions on $\cS$ from~\eqref{eq:effective-dynamics-explicit} and applying Lemma~\ref{lem:locally-Lipschitz}. 
    
    Combining all of the above, we obtain for each $k\le T\delta^{-1} \wedge \tau \wedge \bar \tau$
    \begin{align*}
         \max_{f\in \mathbf{u}} |\Delta_k^f| \le C \delta \sum_{\ell \le k} \max_{f\in \mathbf{u}}|\Delta_\ell^f| +o(1)
    \end{align*}
    for a constant $C$ depending on $R,T$, the Lipschitz constant bounds and all other $O(1)$ quantities.   
    Thus by the discrete Gronwall inequality, there is a constant $C$ such that 
    \[
    \max_{\ell\leq T \delta^{-1} \wedge \tau\wedge\bar\tau}\max_{f\in\bu}|\Delta_\ell^f|\leq o(1)\cdot \exp( C\delta d T) =o(1) \qquad \text{in $L^2$}\,,
    \]
    where we have used here the coupling $\bar \Theta_0=\Theta_0$.

    Consider now the continuous-time linear interpolants, $(\bu_t^d)_t$ of the discrete time $\bu(\Theta_{\lfloor t \delta^{-1}\rfloor})$. Let $\check{\tau}_R$ denote the exit time for $\bu_t^d$ from $B_R(0)$ and $\check{\bar{\tau}}_R$ denote the same for the Gaussian model. Note that $\check\tau_R \delta^{-1}\in[\tau_R-1,\tau_R]$ and similarly for $\check{\bar\tau}_R$. Then, by the above, we deduce
    \[
    \sup_{t\in[0,T]\wedge\check {\tau}_R^d\wedge\check{\bar\tau}_R^d}\norm{\bu_t^d-\bar\bu_t^d}\to0 \qquad \text{in $L^2$}\,.
    \]
    Consequently the stopped processes $\bu^d_{t\wedge\check\tau_R}-\bar\bu^d_{t\wedge\check{\bar\tau}_R}\to 0$ in probability. Since by Theorem~\ref{thm:Gaussian-projective-has-effective-dynamics}, the Gaussian stopped process $\bar \bu^d_{t\wedge\check\tau_R}$ converges to the solution of the desired ODE, the non-Gaussian one $\bu^d_{t\wedge\tau_R}$ must as well. Since this holds for all $R$, by a standard localization argument (Lemmas 11.1.11-12 of~\cite{StroockVaradhan06}), the full process $(\bu^d_t)_{t\ge 0}$ must also converge to the solution of the ODE~\eqref{eq:ballistic-dynamics}. 
\end{proof}

\subsection{Non-universal ballistic dynamics with coordinate aligned initialization} 
\label{subsec:non-universal-ballistic-bad-init}
We end this section with the following example demonstrating the importance of the coordinate-delocalized condition for the ballistic universality result. 
Recall that \emph{smooth phase retrieval} is the single index model with $f(x) = x^2$, i.e., $\theta, \theta_* \in \mathbb R^d$ and features $X = (X_i)_{i=1}^d$ for $X_i \sim \nu$ i.i.d.\ with 
\begin{align}\label{eq:smooth-phase-retrieval-loss}
    L(\theta,X) = | \langle X,\theta\rangle^2 - \langle X,\theta_*\rangle ^2|^2\,.
\end{align}
Also,  let $\theta_* = d^{-1/2} \mathbf 1$ (though any coordinate-delocalized ground truth vector $\theta_* \in {\deloc}_{\sfrac{1}{10}}$ would evidently work). Corollary~\ref{cor:multi-index-regression} shows that if $\theta_0$ is also coordinate-delocalized then the limiting dynamics of the pair of summary statistics $\mathbf{G} = (\theta,\theta_*)^\top (\theta,\theta_*)$ are the same under $\nu$ as under i.i.d.\ $\mathcal N(0,1)$ distribution on the features. The following proposition shows that, by contrast, coordinate-aligned initializations break this universality. 

\begin{prop}\label{prop:no-ballistic-universality-bad-init}
    Consider SGD at learning rate $\delta = c_{\LR}/d$ with respect to~\eqref{eq:smooth-phase-retrieval-loss}, features $X$ with $X_i \sim \nu$ i.i.d.\ for $\nu$ centered, with variance one, all finite moments, and $\E_\nu[X_1^4]\ne 3$. For all except at most one value of $c_{\LR}>0$, one has that 
    if the initialization is $\theta_0 = e_1$, then the summary statistics $\mathbf{G} = (\theta, \theta_*)^\top (\theta,\theta_*)$ do not follow the limit that they do if the feature distribution was i.i.d.\ $\mathcal N(0,1)$.  
\end{prop}

The key to showing the non-universality is showing that at the initialization, the drift under population gradient descent for the summary statistic $\|\theta\|_2^2$ is different under $\nu$ features than for Gaussian $\mathcal N(0,1)$ features. As before, denote the expectation with respect to the Gaussian distribution by $\bar \E$ and that with respect to $\nu$ by $\E$.

\begin{lem}\label{lem:non-universal-ballistic-bad-init}
    Consider SGD with respect to~\eqref{eq:smooth-phase-retrieval-loss} with features $X$ with $X_i \sim \nu$ i.i.d.\ for $\nu$ having $\E[X_1^4]\ne 3$. Then at $\theta_0= e_1$, 
    \begin{equation*}
        \lim_{d\to\infty} \mathbb E[\langle \nabla L(\theta_0),\theta_0\rangle ] \ne \lim_{d\to\infty} \bar \E[\langle \nabla L(\theta_0),\theta_0\rangle ]\,.
    \end{equation*}
\end{lem}

\begin{proof}
We are considering the quantity 
$
\langle \nabla L ,\theta\rangle= 2(\langle X,\theta\rangle ^2 - \langle X,\theta_*\rangle ^2) \langle X,\theta\rangle^2 $ at the points $\theta = e_1$ and $\theta_* = d^{-1/2}\mathbf 1$, whence we are taking expectation under $\nu^{\otimes d}$ or $\mathcal N(0,1)^{\otimes d}$ of 
\[
\langle \nabla L ,e_1\rangle= 2(X_1^4 - X_1^2 \langle X,d^{-1/2}\mathbf 1\rangle ^2) \,.
\]
Taking expectation and limit under i.i.d.\ standard Gaussian $X$, we get 
\[
\bar \E[\langle \nabla L,e_1\rangle  ]  =4+o(1)\,,
\]
while for non-Gaussian, $X$ drawn i.i.d.\ from $\nu$, we get 
\[
\mathbb E[\langle \nabla L,e_1\rangle ] = 2( \E [X_1^4] - d^{-1} \E[X_1^4] - d^{-1} \cdot 1\cdot \sum_{i=2}^d 1) = 2(\E [X_1^4] - 1 -o(1))\,.
\]
The $d\to\infty$ limit of this is not $4$ if $\E [X_1^4]\ne 3$.     
\end{proof}

\begin{proof}[\textbf{\emph{Proof of Proposition~\ref{prop:no-ballistic-universality-bad-init}}}]
    Consider the SGD with noise distribution $\nu$. Suppose that the evolution of  $\mathbf{G}_{\lfloor t \delta^{-1}\rfloor} = (\theta_{\lfloor t\delta^{-1}\rfloor},\theta_*)^\top (\theta_{\lfloor t \delta^{-1}\rfloor},\theta_*)$ admits a $d\to\infty$ limit as otherwise, the statement holds vacuously (because they do admit a limit under the Gaussian features by Theorem~\ref{thm:ballistic-universality}). For  $u_\ell = \|\theta_\ell\|_2^2$,  one has 
    \begin{equation*}
        u_{\ell} = u_0 + \delta \sum_{j\le \ell} \E[ \langle \nabla L(\theta_j),\nabla u_j \rangle] + c_{\LR}\cdot  \frac{\delta}{d}  \sum_{j\le \ell} \cL u_j +  M_\ell\,,
    \end{equation*}
    where $M_\ell$ is a martingale. Taking expectations, it  suffices to show that for $\epsilon>0$ small, there is an $\eta>0$ such that in the first $\eta d$ steps, $\E[\langle \nabla  L(\theta_j),\nabla u_j \rangle ]$ is within $\eps$ of $\E[\langle \nabla L(\theta_0),\nabla u_0\rangle]$. That would yield a macroscopic difference between the Gaussian ballistic limit, as that one has drift coming from its corresponding first term that is $4+o_\eta(1)$ by continuity. The only exception to this is possibly at one choice of $c_{\LR}$ where  the difference in expectations on $\E[c_{\LR} (\frac{1}{d} \delta \sum_{j\le \ell} \mathcal L u_j)]$ could exactly cancel the difference in the first term.

    To show this,  consider the evolution of $\phi(\theta)  = \E[\langle \nabla L,\theta\rangle]$ under SGD, which follows 
    $$\phi(\theta_\ell)  = \phi(\theta_0) + \delta \sum_{j\le \ell} \big(\phi(\theta_j) - \phi(\theta_{j-1})\big)\,.$$
    We wish to establish some continuity for this, namely that it evolves an order one amount in linear number of steps. By the mean value theorem, for every $\theta$ and $X$, there is $\theta'$ such that 
    \begin{equation}\label{eq:mvt-to-get-difference}
        \phi(\theta+ \delta \nabla L) - \phi(\theta) = \delta \langle \nabla L, \nabla \phi\rangle (\theta')\,.
    \end{equation}
    We claim that uniformly over $\theta' \in B_R(0)$, the right-hand side is $O(\delta)$ both in expectation, and in the sense that its $k$'th moment is $O(\delta^k)$. Then when summed over $\ell = \eta d$ many steps, this still contributes at most $\eps$ difference from the initial $\phi_0 = \E[\langle \nabla L,e_1\rangle]$ as claimed, by the martingale law of large numbers. 

    To see the claim that~\eqref{eq:mvt-to-get-difference} is $O(\delta)$, we write out
    \begin{align*}
        \langle \nabla L, \nabla \phi\rangle (\theta)
        & = \tilde \E \Big[4(\langle X,\theta \rangle^2  - \langle X,\theta_*\rangle ^2 ) \langle X,\theta\rangle  ( 2\langle \tilde X,\theta\rangle^3 - \langle \tilde X,\theta\rangle \langle \tilde X,\theta_*\rangle) \langle \tilde X,X\rangle \Big]\,,
    \end{align*}
    where $\tilde X$ is an independent copy of $X$ and $\tilde \E$ is over $\tilde X$ only. 
    This can be expanded out explicitly, in terms of moments of $\tilde X$, e.g., 
    \begin{equation*}
        \tilde \E[\langle \tilde X, \theta\rangle^3 \langle \tilde X,X\rangle] = \tilde \E \Big[ \sum_{j_1,j_2,j_3,k}  \tilde X_{j_1} \tilde X_{j_2}\tilde X_{j_3} \tilde X_k \Big]  \theta_{j_1} \theta_{j_2} \theta_{j_3}  X_k\,.
    \end{equation*}This is only non-zero when the indices form two pairs or one quadruple; so 
    \begin{equation*}
        \tilde \E[ \langle \tilde X,\theta \rangle^3 \langle \tilde X,X\rangle] = \sum_{i,j}  X_i \theta _i \theta_j\theta_j  = \langle X,\theta\rangle \|\theta\|^2\,.
    \end{equation*}
    The other terms in the expansion of $\langle \nabla L,\nabla \phi\rangle$ are similarly bounded by polynomials in $\langle X,\theta\rangle, \langle X,\theta_*\rangle$ and $\|\theta\|^2$. Uniformly over $\theta \in B_R(0)$, all moments of this (in $X$ now) are $O(1)$. 
\end{proof}

\section{Non-universality of diffusive limits}\label{sec:non-universal-diffusive-limit}

We end with Theorem~\ref{thm:no-diffusive-universality} which is an example demonstrating that unlike the ballistic dynamics, the ``diffusive'' scaling limit of the summary statistics about fixed points of their ballistic limits can fail to be universal even with coordinate-delocalized initializations and optimal regularity assumptions.  

To that end, consider the noiseless single-index model loss function: 
\begin{align}\label{eq:no-universality-loss-function}
    L(\theta,X) = | f(\langle X,\theta\rangle) - f(\langle X,\theta_*\rangle)|^2\,. \qquad \text{where} \qquad f(x) = x^3 - 3x +x^2\,.
\end{align}
(This is the same loss function as one gets from link function  $\mathsf{He}_3(x) + \mathsf{He}_2(x)$ because the constant term cancels out.) 
Take $\theta_* = \rho d^{-1/2} \mathbf 1$ for $\rho>0$ and $X$ to be the features with $X_i$ either drawn i.i.d.\ as $\mathcal N(0,1)$ or i.i.d.\ from $\nu$ for $\nu$ having mean zero, variance one, all finite moments, and $\E_{\nu}[X_1^3] =:  \mathfrak{m}_3 \ne 0$. Let $\bar \E$ be expectation with respect to the standard Gaussian.

    By Theorem~\ref{thm:ballistic-universality} and Corollary~\ref{cor:multi-index-regression}, the pair of summary statistics $\mathbf{G}_{\lfloor t \delta^{-1} \rfloor}  = (\theta_{\lfloor t\delta^{-1}\rfloor} ,\theta_*)^\top (\theta_{\lfloor t\delta^{-1} \rfloor},\theta_*)$ admit a common limiting ballistic dynamics. That limiting dynamics has an ``uninformative" fixed point at the summary statistic values $\langle \theta,\theta_*\rangle  =0$ and $\langle \theta,\theta\rangle = R_*(c_{\LR})$ for an $R_*$ that scales down to zero linearly with $c_{\LR}$ (see e.g.,~\cite{Rangriz-SGD-single-index} for the explicit ODE). This is because the link function chosen has information exponent larger than $1$. 

We aim to show that the limit of the rescaled summary statistic $\sqrt{d}\langle \theta,\theta_*\rangle$ about this fixed point is non-universal. 
We show that its drift differs between Gaussian and non-Gaussian noise distributions, for all coordinate-delocalized parameter values $\theta$. We take as initialization $\theta_0 \sim \mathcal N(0,R_*I/d)$ which is a coordinate-delocalized initialization that is uninformative, and it places the initial summary statistic values within $O(d^{-1/2})$ of the fixed point $(0,R_*)$.

Since there are many terms appearing in the difference in expectations, in order to not have to worry about possible cancellations between them, we will expand in orders of $\rho$ and establish a distinction by taking $\rho$ small enough.

\begin{lem}\label{lem:diffusive-drifts-are-different}
Fix $c_{\LR}>0$, let $\theta_* = \rho d^{-1/2} \mathbf 1$, suppose $\theta\in {\deloc}_{\sfrac{1}{10}}$, and suppose it has $\langle \theta,\theta_*\rangle =O(d^{-1/2})$ and $\|\theta\|_2^2 = R_* + O(d^{-1/2})$. For $\rho>0$ small, 
    \begin{equation*}
        \sqrt{d} | \E_{\nu} [ \langle \nabla L,\theta_*\rangle] - \bar \E[\langle \nabla L,\theta_*\rangle]| = 18\mathfrak{m}_3 \rho R_* + O(\rho^2) + o(1)\,.
    \end{equation*}
\end{lem}
\begin{proof}
    We are considering the difference in Gaussian and $\nu$ expectations of $$\E[\langle\nabla L,\theta\rangle]=2\E\left[\left(f(\langle \theta, X\rangle)-f(\langle\theta_{*}, X\rangle)\right)f'(\langle\theta, X\rangle)(\langle\theta_*, X\rangle)\right]\,.$$ We write this as 
    \begin{equation*}
        2\Delta [ (f(x)-f(y))f'(x)y]
    \end{equation*}
    where $\Delta = \bar \E - \E_\nu$ is the difference of expectations operator, $x= \langle X,\theta\rangle$ and $y = \langle \theta_*,X\rangle$. 
    Note  
    \begin{align}\label{eq:f-ff'y-expansion}
        (f(x) - f(y))f'(x)y & = 3x^5 y + 5x^4 y - 10x^3 y -3x^2 y^4 -3x^2 y^3 + 9x^2 y^2  -2xy^4 - 2xy^3 +3y^4 \nonumber \\
         & \qquad + 6xy^2  +3y^3 - 9x^2 y -9y^2 +9xy \,.
    \end{align}
    The terms that are of degree at most two vanish under the difference of expectations operator because the first two moments of $\nu$ and $\mathcal N(0,1)$ match. All the remaining terms can be expanded out and seen to be of order $O(d^{-1/2})$, and it is straightforward to see that any terms that entail powers of strictly more than one of  $y$ scale like $O(\rho^2 d^{-1/2})$. Therefore, the only contributions that are not $O(\rho^2)$ or $o(1)$ when multiplied by $\sqrt{d}$ are those coming from $x^2y, x^3y, x^4 y, x^5 y$: we get 
    \begin{align*}
            \Delta [x^2 y] & = \Delta [\sum_{i_1,i_2} X_{i_1}X_{i_2}  \theta_{i_1} \theta_{i_2}\sum X_{j} \theta^*_j] = \mathfrak{m}_3 \sum_i \theta_i^2 \theta_{*,i} = \mathfrak{m}_3 \rho d^{-1/2} R_* \\
        \Delta[x^3 y] & = \Delta [\sum_{i} X_i^4 \theta_i^3 \theta_{*,i} ]=  \|\theta\|_\infty \langle \theta_*,\theta\rangle   = o(d^{-1/2}) 
    \end{align*}
    because $\theta$ is coordinate-delocalized for $\zeta<1/2$. The other two are handled analagously, though the expansions are a bit longer, so we just write their conclusion, which is that because $\theta$ is coordinate-delocalized, these are also little-o of the $x^2 y$ term, and in particular, 
    \begin{equation*}
        \Delta [x^4 y ] \vee \Delta [x^5 y] = o(d^{-1/2})\,.
    \end{equation*}
    Combining the above, we get the claimed bound. 
\end{proof}

\begin{proof}[\textbf{\emph{Proof of Theorem~\ref{thm:no-diffusive-universality}}}]
    Suppose that the pair of summary statistics $\widetilde{\mathbf{G}} = (\sqrt{d}\langle\theta,\theta_*\rangle, \sqrt{d}(\|\theta\|_2^2 - R_*))$ have an SDE limit on linear timescales, as otherwise the claim that the limit is not the Gaussian one holds vacuously. Consider the summary statistic $\tilde u(\theta) = \sqrt{d} \langle \theta,\theta_*\rangle$. If $\tilde u_\ell = \tilde u(\theta_\ell)$,
        \begin{equation*}
        \tilde u_\ell  = \tilde u_0 + \sqrt{d}\delta \sum_{j\le \ell} \langle \nabla L(\theta_j,X),\theta_{*}\rangle\,.
    \end{equation*}
    Taking expectations, we will show that for all $c_{\LR}>0$, for all sufficiently small $\rho>0$, for a small linear number of steps, $\ell= \eta d$, the difference of the Gaussian vs.\ $\nu$ expectation of 
    \begin{equation*}
        \sqrt{d} \langle \nabla L(\theta_j,X) ,\theta_*\rangle = \sqrt{d} 2 (f(\langle X,\theta_j\rangle ) - f(\langle X,\theta_*\rangle) ) f'(\langle X,\theta_j\rangle) \langle X,\theta_*\rangle 
    \end{equation*}
    is bounded away from zero. By Theorem~\ref{thm:stays-coordinate-delocalized} with high probability, for all linear timescales $\theta_j$ is in ${\deloc}_{\sfrac{1}{8}}$. As a result, by Lemma~\ref{lem:diffusive-drifts-are-different}, for any fixed $c_{\LR}>0$ and $\nu$ having $\mathfrak{m}_3 \ne 0$, for $\rho$ sufficiently small, for each $j\le \eta d$ the difference in expectations $\Delta [ \sqrt{d}\langle \nabla L_j ,\theta_*\rangle]$ is at least some $\epsilon>0$.  
    This contributes at least a $c_{\LR} \eta \epsilon>0$ to the difference in the expectations of $\tilde u_{\eta d}$ under Gaussian and $\nu$ feature distributions. 
\end{proof}

\bibliographystyle{plain}
\bibliography{references}

\end{document}